\PassOptionsToPackage{table,dvipsnames}{xcolor}
\PassOptionsToPackage{numbers,sort&compress}{natbib}
\documentclass{article} %
\usepackage{iclr2027_conference,times}

\usepackage{amsmath,amsfonts,bm}

\def\eqref#1{equation~\ref{#1}}

\def\1{\bm{1}}

\DeclareMathAlphabet{\mathsfit}{\encodingdefault}{\sfdefault}{m}{sl}
\SetMathAlphabet{\mathsfit}{bold}{\encodingdefault}{\sfdefault}{bx}{n}

\usepackage{url}
\usepackage{wrapfig}
\usepackage{tikz}
\usepackage{amsmath,amssymb,amsfonts,amsthm}
\usepackage{algorithm}
\usepackage{graphicx}
\usepackage{textcomp}
\usepackage{bm}
\usepackage{algpseudocode}
\usepackage{fontawesome5}
\newcommand{\taskimage}[2]{%
\begin{tikzpicture}
    \node[anchor=south west, inner sep=0] (image) at (0,0) {%
        \includegraphics[width=\linewidth]{#1}%
    };
    \node[
        anchor=north west,
        font=\bfseries\scriptsize,
        fill=white,
        fill opacity=0.8,
        text opacity=1,
        inner sep=1.5pt
    ] at ([xshift=2pt,yshift=-2pt]image.north west) {#2};
\end{tikzpicture}%
}

\usepackage{booktabs}
\usepackage{multirow}
\usepackage{makecell}
 \usepackage{subcaption}
\usepackage[table,dvipsnames]{xcolor}
\definecolor{LightBlue}{HTML}{C9E9F6}

\newcommand{\ourscell}[1]{\cellcolor{LightBlue}#1}

\definecolor{CiteBlue}{RGB}{30, 90, 180}
\definecolor{RefRed}{RGB}{170, 40, 60}
\definecolor{UrlTeal}{RGB}{0, 130, 130}

\definecolor{myblue}{HTML}{71B5F5}
\definecolor{myred}{HTML}{F77189}
\usepackage{pifont}

\usepackage[
    colorlinks=true,
    citecolor=CiteBlue,
    linkcolor=RefRed,
    urlcolor=myblue,
    bookmarks=true
]{hyperref}
\usepackage{cleveref}

\Crefname{figure}{Figure}{Figures}
\Crefname{figure}{Figure}{Figures}

\newtheorem{proposition}{Proposition}

\newcommand{\fhat}{\widehat{f}}
\newcommand{\policy}{\pi_{\mathrm{pre}}}

\def\BibTeX{{\rm B\kern-.05em{\sc i\kern-.025em b}\kern-.08em
    T\kern-.1667em\lower.7ex\hbox{E}\kern-.125emX}}

\newcounter{todos}
\AtEndDocument{\ifnum\value{todos}>0 \PackageWarning{TODOS}{There are \arabic{todos} todos left in this paper! Fix them before submitting the paper!} \fi}

\title{\textsc{DiVeR}: Decision-Critical Verifier Learning for VLA Test-Time Scaling}

\ificlrfinal
\else
\fi

\author{} 

\usepackage{diver-msr-layout}
\renewcommand{\reportauthors}{%
  \mbox{Seongheon Park\textsuperscript{1,2,*}}\quad
  \mbox{Heecheol Kim\textsuperscript{2}}\quad
  \mbox{Shulin Tian\textsuperscript{3}}\quad
  \mbox{Lilika Makabe\textsuperscript{4}}\\[1mm]
  \mbox{Namiko Saito\textsuperscript{2}}\quad
  \mbox{Katsushi Ikeuchi\textsuperscript{5}}\quad
  \mbox{Sharon Li\textsuperscript{1}}\quad
  \mbox{Yasuyuki Matsushita\textsuperscript{2}}}
\newcommand{\affiliationitem}[3]{%
  \mbox{\textsuperscript{#1}\hspace{0.4mm}%
    \raisebox{-0.8mm}{\includegraphics[width=3.5mm,height=3.5mm]{#2}}%
    \hspace{0.7mm}#3}}
\renewcommand{\reportaffiliations}{%
  \affiliationitem{1}{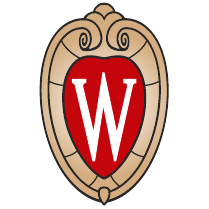}{University of Wisconsin--Madison}\hspace{3mm}%
  \affiliationitem{2}{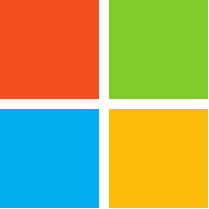}{Microsoft Research Asia--Tokyo}\\[1.0mm]
  \affiliationitem{3}{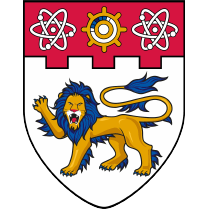}{Nanyang Technological University}\hspace{3mm}%
  \affiliationitem{4}{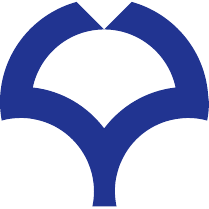}{The University of Osaka}\hspace{3mm}%
  \affiliationitem{5}{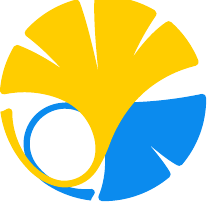}{The University of Tokyo}}
\renewcommand{\reportauthornote}{%
  \textsuperscript{*}Work done during an internship at Microsoft Research Asia--Tokyo.}
\renewcommand{\reportcorrespondence}{%
  Seongheon Park (\href{mailto:seongheon_park@cs.wisc.edu}{\nolinkurl{seongheon_park@cs.wisc.edu}})}

\begin{document}

\addtocontents{toc}{\protect\setcounter{tocdepth}{-1}}

\maketitle

\begin{abstract}

Scaling robot data and model capacity has improved Vision-Language-Action (VLA) policies, but further progress is constrained by the high cost of robotic data.
Verifier-guided test-time scaling offers an efficient alternative by sampling multiple action candidates and selecting the one most likely to lead to task success at inference time.
Existing classification-based verifiers learn from trajectory-level outcomes but treat all visited states equally, even though their value for candidate discrimination can vary across a trajectory.
At many states, plausible actions are similar and provide limited discrimination signal, while only a sparse set of decision-critical states admits meaningfully different actions that can substantially affect downstream outcomes.
To address this, we propose \textsc{DiVeR}, which estimates decision criticality from the dispersion of sampled action representations.
\textsc{DiVeR} then uses this signal to reweight verifier learning toward states where action selection is most consequential, without requiring step-level annotations or additional environment interaction.
Across LIBERO, RoboCasa, and real-world experiments on a Franka Research 3 robot, \textsc{DiVeR} consistently improves task success through more effective verifier-guided action selection, while adding negligible verifier inference overhead.
Videos are available at:
\href{https://www.microsoft.com/en-us/research/publication/diver-decision-critical-verifier-learning-for-vla-test-time-scaling/}{\faRobot\ Project Page}.
\end{abstract}

\section{Introduction}
\label{sec:intro}

Vision-Language-Action (VLA) models~\citep{kim2024openvla,black2024pi_0,intelligence2025pi_,intelligence2025pi,bjorck2025gr00t} learn generalist robot policies from large-scale multimodal demonstrations, but further scaling remains limited by the high cost of robotic data collection~\citep{ma2026survey}.
Test-Time Scaling (TTS)\footnote{We use test-time scaling to denote parallel scaling  via repeated sampling and ranking multiple action candidates, rather than sequential scaling via extended rollouts or iterative refinement~\citep{madaan2023self}.}~\citep{snell2024scaling,jang2025verifier,kwok2025robomonkey} offers a complementary route to improving VLA performance without collecting new demonstrations or retraining the base policy. Rather than executing a single sampled action, a pretrained VLA can be viewed as a stochastic generator that proposes multiple candidate action chunks at each decision step. Since this distribution may contain both successful and unsuccessful continuations, a verifier can rank the candidates and select the one most likely to make task progress. This naturally leads to verifier-guided Best-of-\textit{N} (BoN) sampling, a generator--verifier framework that converts additional inference-time computation into better action selection.

\begin{figure}[t!]
\centering
\includegraphics[width=\columnwidth]{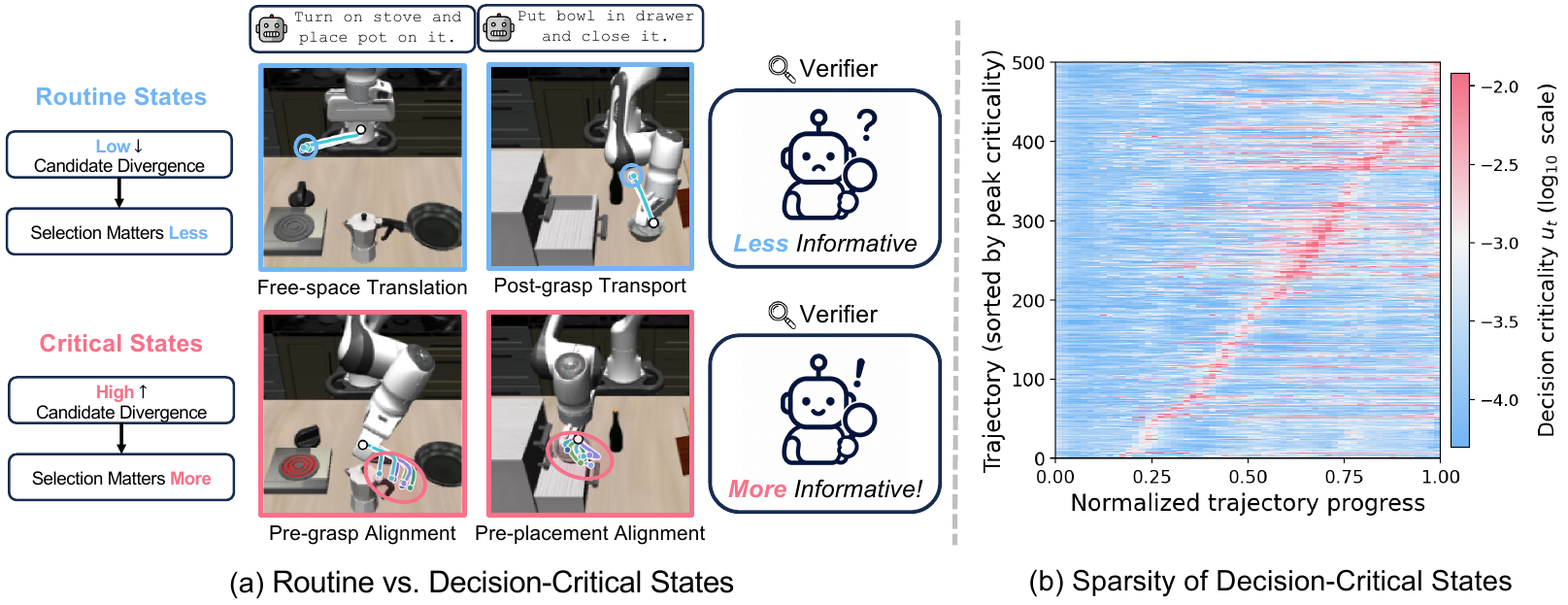}
\caption{
\textbf{Not all states matter equally for verifier learning.}
(a) At \textbf{\textcolor{myblue}{routine states}}, sampled action candidates remain similar, so distinguishing among them is less important for downstream behavior.
At \textbf{\textcolor{myred}{decision-critical states}}, candidates diverge more strongly, making accurate selection more consequential for task outcomes.
(b) Across 500 LIBERO-Long trajectories from a frozen $\pi_0$ policy,
states with \textbf{\textcolor{myred}{high decision criticality}}, measured by $u_t$~(\Cref{eq:action_uncertainty}),
are sparse and temporally localized.
These observations motivate focusing verifier learning on sparse decision-critical states where action selection matters most.
}
\label{fig:1}
\vspace{-5pt}
\end{figure}

A practical approach to verifier learning is to train lightweight classification-based verifiers~\citep{yang2025steering,tsao2026learning,attarian2026update} from trajectory-level success and failure outcomes, avoiding costly step-level supervision~\citep{zhang2026progress}  and the inference overhead of external VLM-based judges~\citep{liang2026robometer,kwok2025robomonkey,kwok2026llm}.
However, existing approaches propagate each trajectory outcome uniformly across all visited states or state–action pairs, implicitly treating them as equally informative for candidate discrimination.
This assumption can be problematic in long horizon robotic execution, where the informativeness of trajectory-level supervision for candidate discrimination can vary substantially across states.

Our key observation is that \textit{not all states are equally consequential for action selection}.
As illustrated in~\Cref{fig:1}(a), at states such as free-space motion or post-grasp transport, sampled action candidates remain highly similar. We refer to these as \textbf{\textcolor{myblue}{routine states}}, where candidate selection is relatively inconsequential and provides limited signal for learning to discriminate among actions. 
In contrast, at \textbf{\textcolor{myred}{decision-critical states}}, such as pre-grasp or pre-placement alignment, plausible candidates diverge more strongly, and small differences in action choice can substantially affect downstream outcomes. 
\Cref{fig:1}(b) further shows that such high-criticality states are sparse and temporally localized rather than uniformly distributed throughout a trajectory.
This pattern echoes findings in language generation, where only a sparse subset of high-entropy minority tokens can lead to divergent continuations and steer generation toward distinct downstream trajectories~\citep{wang2026beyond}.
Together, these observations suggest that uniformly weighting all visited states can dilute the supervision signal from the relatively few states where accurate candidate discrimination matters most.

In this paper, we introduce \textsc{DiVeR}, a
\textbf{D}ecision-cr\textbf{i}ticality-weighted \textbf{VeR}ifier for VLA test-time scaling that focuses learning on states where action selection matters most. 
Our key idea is to estimate decision criticality from disagreement among plausible actions under the frozen VLA policy. At each visited state, we sample multiple action candidates and measure the dispersion of their action-expert representations. Low dispersion characterizes routine states where the policy proposes similar behaviors, whereas high dispersion identifies states where plausible actions diverge and candidate selection can substantially affect downstream outcomes.
This provides a practical proxy for decision criticality without requiring step-level supervision or additional environment interaction.
\textsc{DiVeR} uses this signal to reweight the verifier objective, concentrating trajectory-level supervision on states where accurate candidate discrimination is most consequential.

We evaluate \textsc{DiVeR} on two simulated manipulation benchmarks and a real-world robotic platform using $\pi_0$ and $\pi_{0.5}$.
Across policies and environments, \textsc{DiVeR} improves task success by up to \textbf{+13.6\%} over single sample inference and \textbf{+3.4\%} over the strongest baseline.
Notably, \textsc{DiVeR} also surpasses the VLM-based verifier by \textbf{+4.6\%} success rate while reducing verifier inference latency by over $\mathbf{700\times}$.
These results demonstrate the effectiveness of focusing verifier learning on the sparse states where candidate differences matter most.

Our key contributions are summarized as follows:
\begin{itemize}
\item We identify a limitation of uniform verifier training, showing that routine states dominate trajectories while providing limited signal for candidate discrimination, whereas sparse decision-critical states are more consequential for action selection.
\item We propose \textsc{DiVeR}, a novel decision-criticality-weighted verifier that uses action-expert representation variance as a practical proxy for decision criticality and places greater emphasis on states with stronger candidate disagreement.
\item We validate \textsc{DiVeR} across simulated and real-world robotic manipulation tasks, demonstrating consistent improvements in task success through more effective verifier-guided Best-of-\(N\) action selection, with negligible verifier overhead at inference.
\end{itemize}

\section{Related Works}
\label{sec:relaetd}

\textbf{Vision-Language-Action (VLA) models} learn generalist robot policies by fine-tuning pretrained vision-language backbones on large-scale robot demonstration data~\citep{ma2026survey, kim2024openvla, bjorck2025gr00t, black2024pi_0, barreiros2025careful}.
Recent approaches
further improve these policies along several complementary directions,
including geometry-aware representations that enhance 3D spatial understanding~\citep{qu2025spatialvla, li2026pointvla}, memory mechanisms for long-horizon tasks~\citep{shi2025memoryvla, torne2026mem, sridhar2025memer}, and world models that capture
environment dynamics~\citep{ye2026world, cen2025rynnvla, gao2026dreamdojo}. While effective, these directions
strengthen the policy itself and thus require additional data collection,
architectural modification, or retraining. In contrast, we treat the
pretrained VLA as a frozen stochastic action generator and improve its
execution purely at inference time, making our approach complementary to
these advances.

\textbf{Test-time Scaling in LLMs}
improves performance by allocating additional computation at inference time rather than through further training, e.g., via repeated sampling with majority voting~\citep{wang2022self,brown2024large,kang2026scalable} or verifier-guided Best-of-\textit{N} selection with outcome- and process-level reward models~\citep{snell2024scaling,zhang2025lessons, oh2026neglected, kang2026mid}.
Recent studies further suggest that not all generation steps are equally consequential: a small subset of high-entropy \textit{forking tokens} can act as branching points that steer reasoning toward different continuations, while many low-entropy tokens largely follow an already determined trajectory~\citep{bigelow2025forking,wang2026beyond,al2026fork}.
Focusing supervision on these forking tokens can therefore yield disproportionate gains, suggesting that the value of supervision may vary across a trajectory.
We study an analogous non-uniformity in robotic control through \textit{decision-critical states}, where candidate action selection can have substantially different downstream consequences.
Accordingly, we focus verifier learning on these sparse states rather than treating all states uniformly.

\textbf{Test-time Scaling in VLAs.}
Motivated by advances in LLM test-time scaling, recent works have adopted
generate-then-verify paradigms for robotic control.
For example, (1) VLM-based verifiers score candidate actions
with large vision-language models~\citep{kwok2025robomonkey, kwok2026scaling, dai2025rover, liang2026robometer, duan2024aha}, but
their inference cost and deployment overhead hinder real-time control;
(2) Verifier-free methods instead rely on the policy's internal signals, such
as distributional confidence or candidate density~\citep{jang2025verifier, rosasco2025kdpe, choi2026scale, park2026acg},
yet these signals are not grounded in task outcomes; 
(3) Classification-based verifiers~\citep{yang2025steering,attarian2026update,tsao2026learning} offer a lightweight alternative by learning from trajectory-level outcomes, but weight all visited states uniformly.
Our approach instead accounts for the non-uniform importance of states by
estimating decision criticality from action-expert representation variance and
placing greater training emphasis on states where candidate actions diverge
and selection is most consequential.

\section{Problem Setup}
\label{sec:prelim}

\textbf{Verifier-guided Best-of-\(N\) action selection.}
We consider a generator--verifier framework for improving a generalist VLA policy through test-time compute.
We model robotic execution over decision steps \(t\in[T]\), with policy-conditioning state space \(\mathcal{S}\) and action space \(\mathcal{A}\), where \(T\) is the trajectory horizon.
The state \(s_t\in\mathcal{S}\) includes the current visual observation, language instruction, and any additional context used by the policy.
Conditioned on \(s_t\), the pretrained VLA policy \(\policy(\cdot\mid s_t)\) generates an action chunk
\(\mathbf{a}_t=(a_{t,1},\ldots,a_{t,H})\in\mathcal{A}^{H}\),
where \(H\) is the action-chunk horizon.
Rather than executing a single prediction, we treat the policy as a stochastic generator and sample \(N\) candidate action chunks,
\begin{equation}
\mathbf{a}^{(i)}_t
\sim
\policy(\cdot\mid s_t),
\qquad i\in[N].
\end{equation}
A learned verifier
\(\fhat:\mathcal{S}\times\mathcal{A}^{H}\rightarrow\mathbb{R}\)
assigns a scalar score to each candidate, and the robot executes
\begin{equation}
\mathbf{a}^{\star}_t
=
\arg\max_{i\in[N]}
\fhat(s_t,\mathbf{a}^{(i)}_t).
\label{eq:bon}
\end{equation}
Thus, the VLA policy serves as a stochastic proposer of plausible robot behaviors, while the verifier provides an inference-time selection rule for choosing the candidate most likely to satisfy the instruction.

\section{Method}
\label{sec:method}

\begin{figure*}[t!]
\centering
\includegraphics[width=\textwidth]{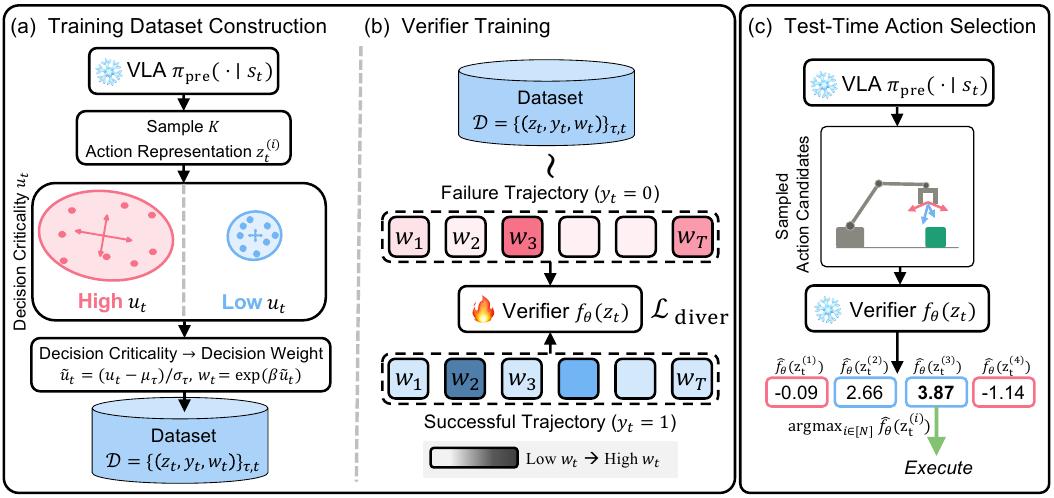}
\vspace{-15pt}
\caption{
\textbf{Overview of \textsc{DiVeR}.}
(a)~At each visited state, we sample $K$ action candidates from the frozen VLA policy and estimate decision criticality $u_t$ from the dispersion of their action-expert representations.
The scores are standardized within each trajectory and converted into decision weights $w_t$, which are paired with the executed action representation $z_t$ and trajectory outcome $y_t$ to construct the training dataset $\mathcal{D}$.
(b)~The verifier is trained with decision-weighted binary cross-entropy, assigning greater importance to states with larger $w_t$.
(c)~At inference, the verifier scores $N$ sampled action candidates at each decision step and executes the highest-scoring action chunk.
}

\label{fig:pipeline}
\end{figure*}

\textbf{Overview.}
\textsc{DiVeR} learns a lightweight verifier from successful and failed trajectories while emphasizing states where candidate action selection is most consequential.
We estimate decision criticality from the dispersion of action-expert representations sampled from the frozen VLA policy and use this signal to weight each state's contribution to verifier learning.
The overall framework is illustrated in~\Cref{fig:pipeline}.

\subsection{Success-Visitation Verifier}
\label{sec:success_visitation}

We formulate the verifier as a lightweight discriminator using trajectory-level success and failure outcomes.
Each trajectory provides a sequence of state--action pairs together with a terminal outcome label.
Let $\rho^{+}(s,\mathbf{a})$ and $\rho^{-}(s,\mathbf{a})$ denote the state--action visitation densities induced by successful and failed trajectories, respectively.
The discriminator $f_\theta(s,\mathbf{a})\in(0,1)$ distinguishes samples from these two visitation distributions and is optimized using binary cross-entropy:

\begin{equation}
\mathcal{L}_{\mathrm{BCE}}(\theta)
=
-\mathbb{E}_{(s,\mathbf{a})\sim\rho^{+}}
    \left[\log f_\theta(s,\mathbf{a})\right]
-\mathbb{E}_{(s,\mathbf{a})\sim\rho^{-}}
    \left[\log(1-f_\theta(s,\mathbf{a}))\right].
\end{equation}

Under balanced sampling between successful and failed trajectories, the population-optimal discriminator is
\begin{equation}
f^{\star}(s,\mathbf{a})
=
\frac{\rho^{+}(s,\mathbf{a})}
{\rho^{+}(s,\mathbf{a})+\rho^{-}(s,\mathbf{a})}.
\end{equation}
Its logit therefore satisfies
\begin{equation}
\log
\frac{f^{\star}(s,\mathbf{a})}
{1-f^{\star}(s,\mathbf{a})}
=
\log
\frac{\rho^{+}(s,\mathbf{a})}
{\rho^{-}(s,\mathbf{a})}.
\end{equation}

Thus, at the population optimum, the discriminator logit recovers the log ratio between successful and failed state--action visitation densities.
We define $
R(s,\mathbf{a})
=
\frac{\rho^{+}(s,\mathbf{a})}
{\rho^{-}(s,\mathbf{a})}
$
as the success-visitation ratio.
A larger $R(s,\mathbf{a})$ indicates that a state--action pair is more characteristic of successful than failed trajectories, providing a natural signal for candidate ranking.
Accordingly, the discriminator logit serves as the verifier score:
\begin{equation}
\hat{f}_{\theta}(s,\mathbf{a})
:=
\log
\frac{f_{\theta}(s,\mathbf{a})}
{1-f_{\theta}(s,\mathbf{a})}.
\end{equation}

In practice, rather than operating directly on raw state--action pairs, we represent each candidate using an internal hidden representation $z_t\in\mathbb{R}^{d}$ from the VLA policy's action expert~\citep{gu2025safe,yang2025steering}.
We parameterize the discriminator as $f_\theta(z_t)$ and use its logit as the verifier score.
In this representation space, the same formulation estimates the log density ratio between representations induced by successful and failed trajectories.

\textbf{Do all states matter equally?}
Although this formulation provides a lightweight outcome-grounded verifier, it treats all visited states equally during training.
However, the usefulness of each state for candidate discrimination can vary 
along a trajectory.
At \textcolor{myblue}{\textbf{routine states}}, successful and failed trajectories often visit similar states and execute similar actions, providing limited discrimination signal.
In contrast, at \textcolor{myred}{\textbf{decision-critical states}}, plausible actions can diverge and lead to substantially different downstream outcomes.
Because routine states dominate trajectories (\Cref{fig:1}(b)), uniform weighting can dilute the signal from the sparse states where accurate candidate selection matters most.
This motivates estimating decision criticality and placing greater emphasis on these consequential states during verifier learning.

\subsection{Decision Criticality from Action-Representation Variance}
\label{sec:decision_criticality}

\begin{wrapfigure}{r}{0.45\columnwidth}
    \centering
    \vspace{-12pt}
\includegraphics[width=\linewidth]{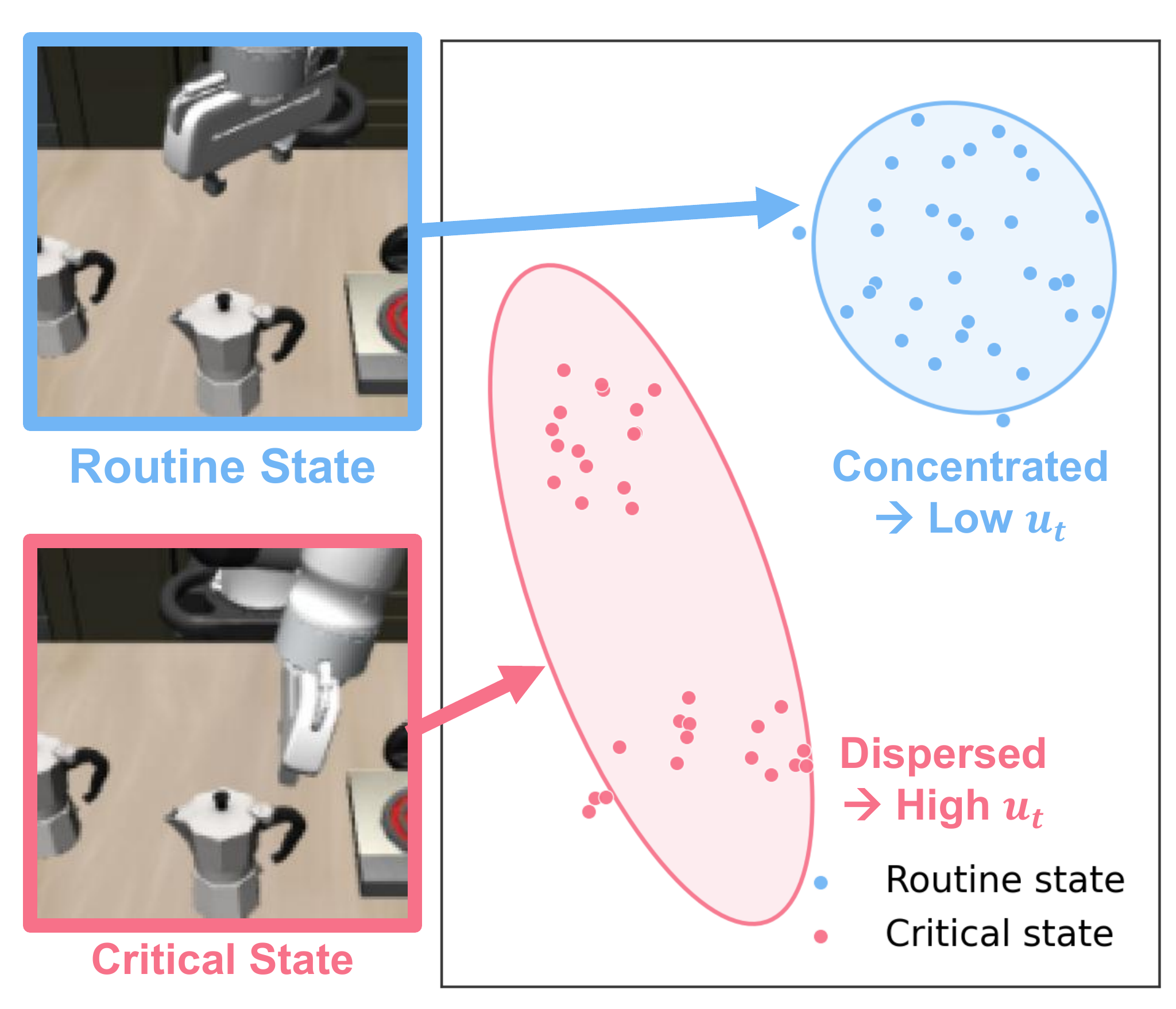}
    \vspace{-20pt}
    \caption{
\small
\textbf{Sampled action representations.}
Candidate representations are concentrated at a routine state
(e.g., free-space motion) but more dispersed at a
decision-critical state (e.g., before grasping).
}
\label{fig:repr_variance}
    \vspace{-5pt}
\end{wrapfigure}

We seek a practical measure of decision criticality that captures how strongly
the policy's plausible actions diverge at a given state, without requiring
step-level supervision or additional environment interaction.
As illustrated by the examples in~\Cref{fig:repr_variance},
candidate action representations can remain tightly concentrated at routine
states while becoming more dispersed at decision-critical states.
We therefore use candidate disagreement in the frozen VLA policy's internal
action-representation space as a proxy for decision criticality.
For each state $s_t$, we sample $K$ candidate action chunks from the frozen policy,
\begin{equation}
\widetilde{\mathbf{a}}_t^{(k)}
\sim
\policy(\cdot\mid s_t),
\qquad k\in[K],
\end{equation}
and extract the corresponding internal hidden representation from the VLA action expert\footnote{Details of the representation extraction are provided
in Appendix~\ref{sec:experimental_details}.}
$z_t^{(k)}\in\mathbb{R}^{d}$, where
$z_{t,j}^{(k)}$ denotes its $j$-th coordinate.
We quantify candidate disagreement by the mean per-coordinate variance,
\begin{equation}
u_t
=
\frac{1}{d}
\sum_{j=1}^{d}
\operatorname{Var}_{k\in[K]}
\left[z_{t,j}^{(k)}\right]
=
\frac{1}{d}
\operatorname{Tr}
\left(
\operatorname{Cov}_{k\in[K]}
\left[z_t^{(k)}\right]
\right).
\label{eq:action_uncertainty}
\end{equation}
We use $u_t$ as the decision-criticality score.
Lower values of $u_t$ indicate more concentrated candidate actions, characterizing routine states where there is limited opportunity for meaningful candidate discrimination.
Conversely, higher values reflect stronger disagreement among plausible actions, serving as a proxy for decision-critical states where accurate candidate selection may be more consequential. We further verify in Appendix~\ref{sec:criticality_validation} that $u_t$ correlates with rollout-based oracle criticality.

\subsection{\textsc{DiVeR}: Decision-Weighted Verifier Learning}
\label{sec:weighted_verifier}

\textbf{Decision-weighted training dataset.}
Let \(\mathcal{T}\) denote the set of offline trajectories collected by the frozen VLA policy.
Each trajectory \(\tau\in\mathcal{T}\) consists of a sequence of visited states and executed actions with a terminal success or failure label.
For each timestep \(t\) in trajectory \(\tau\), we store the action-expert representation \(z_t\) of the executed action chunk and compute its decision-criticality score \(u_t\) from \(K\) sampled candidates according to~\Cref{eq:action_uncertainty}.
To obtain a relative measure of criticality within each trajectory, we standardize the scores and convert them into positive decision weights using exponentiated weighting~\citep{peng2019advantage}:
\begin{equation}
\tilde{u}_t
=
\frac{u_t-\mu_{\tau}}{\sigma_{\tau}},
\quad
w_t
=
\exp(\beta\tilde{u}_t),
\quad
\beta \ge 0,
\label{eq:decision_weight}
\end{equation}
where \(\mu_{\tau}\) and \(\sigma_{\tau}\) denote the mean and standard deviation of the criticality scores in trajectory \(\tau\), and \(\beta\) is the weight temperature.
This transformation assigns larger weights to states with above-average candidate disagreement while downweighting more routine states.
The resulting verifier-training dataset is
\begin{equation}
\mathcal{D}
=
\left\{
(z_t,y_t,w_t)
\right\}_{\tau\in\mathcal{T},\, t\in[T_\tau]},
\end{equation}
where \(T_\tau\) denotes the length of trajectory \(\tau\), and \(y_t\in\{0,1\}\) denotes the trajectory-level outcome propagated to timestep \(t\), with \(y_t=1\) for success and \(y_t=0\) for failure.

\textbf{Decision-weighted verifier objective.}
We train the verifier using weighted binary cross-entropy:
\begin{equation}
\mathcal{L}_{\mathrm{diver}}(\theta)
=
-
\mathbb{E}_{(z_t,y_t,w_t)\sim\mathcal{D}}
\left[
w_t
\left(
y_t \log f_\theta(z_t)
+
(1-y_t)\log\left(1-f_\theta(z_t)\right)
\right)
\right].
\label{eq:weighted_bce}
\end{equation}

The decision weight $w_t$ is computed solely from candidate disagreement under
the frozen policy and does not directly use the trajectory outcome $y_t$.
The weighting changes the relative contribution of training samples,
emphasizing states with greater candidate disagreement where candidate
discrimination is potentially more useful.
The overall algorithm is summarized in~\Cref{alg:1}.

\section{Experiments}
\label{sec:experiments}

\subsection{Experimental Setup}
\label{sec:setup}

We evaluate our method in two simulated environments and on a real-world robotic platform.
Across all environments, we adopt a \textit{multi-task setting}, training and evaluating \textit{a single verifier shared across all tasks} within each benchmark.
 Results are averaged over three random seeds. 
Implementation details are provided in Appendix~\ref{sec:experimental_details}.

\textbf{Simulation environments.}
We evaluate on \textbf{LIBERO-Long}~\citep{liu2023libero}, which comprises 10 long-horizon manipulation tasks with 50 episodes per task, and \textbf{RoboCasa}~\citep{nasiriany2026robocasa}, which contains 18 atomic manipulation tasks with 50 episodes per task.
For both environments, we use 30 episodes per task for training and the remaining 20 episodes for evaluation.

\textbf{Real-world environment.}
We conduct real-world experiments on a \textbf{Franka Research 3} across four manipulation tasks.
We fine-tune $\pi_{0.5}$ with LoRA~\citep{hu2022lora} using 40 human teleoperation demonstrations per task and evaluate on 24 episodes per task.
The tasks are: (1) placing a green block in a bowl, (2) placing a cup on a plate, (3) stacking the red block on the blue block, and (4) closing a drawer.

\textbf{Models.} 
We use the representative flow-based VLA models $\pi_0$~\citep{black2024pi_0} and $\pi_{0.5}$~\citep{intelligence2025pi_} as our primary evaluation policies.
Note that \textsc{DiVeR} only requires the policy to generate multiple action candidates and expose an internal action representation, making it applicable beyond flow-matching architectures.
We evaluate the autoregressive OpenVLA~\citep{kim2024openvla} in Appendix~\ref{sec:autoregressive}.
For the flow-based models, we generate candidates by independently sampling the initial Gaussian noise while keeping the policy conditioning fixed, and reuse VLM computation through KV caching.
We report action-sampling latency in Appendix~\ref{sec:infer_action}.

\textbf{Baselines.}
We compare our method with two verifier-free selection methods and two learned verifiers under the same Best-of-$N$ setting, using identical candidate sets sampled from the same frozen policy.
\textbf{KDPE}~\citep{rosasco2025kdpe} selects the candidate with the highest kernel-density estimate, while \textbf{MG-Select}~\citep{jang2025verifier} ranks candidates using condition-masking confidence.
\textbf{TACO}~\citep{yang2025steering} learns a pseudo-count-based score that favors in-distribution state--action pairs, whereas \textbf{SVM}~\citep{tsao2026learning} discriminates successful from failed visitation using trajectory-level outcomes.
Following the original implementation, we concatenate a CNN-based representation of the current image observation with the corresponding action chunk and train an MLP discriminator on the resulting feature.
Details of the baselines are provided in Appendix~\ref{sec:baselines}.

\subsection{Main experiments}
\label{sec:main}

\paragraph{Simulation environment.}
\label{sec:sim}
In~\Cref{tab:libero}, we compare our method with verifier-free and verifier-based baselines on LIBERO-Long while varying the number of action candidates sampled at test time.
Our method achieves robust performance across both models and candidate budgets.
For $\pi_0$ with $N=32$, our method improves the success rate by +5.8\% over the $N=1$ baseline and by +2.8\% over the best-performing baseline.
In~\Cref{tab:robocasa}, we report category-wise performance on RoboCasa atomic tasks with $N=16$.
Our method improves performance across diverse task categories, rather than being biased toward a particular type of task.
On average, it improves success rate by \textbf{+5.9\%} for $\pi_0$ and \textbf{+11.6\%} for $\pi_{0.5}$ over the corresponding $N=1$ baselines.
These results demonstrate the effectiveness of our method across different policies, candidate budgets, and task settings.

\paragraph{Real-robot platform.}
\label{sec:real}
As shown in~\Cref{tab:real_robot}, we evaluate on a Franka Research 3 robot with $N=16$ and compare \textsc{DiVeR} against the $N=1$ $\pi_{0.5}$ policy (Base) and a uniformly trained verifier (Uniform).
Across all four tasks, \textsc{DiVeR} consistently improves performance, achieving an average success rate of \(\mathbf{71.9\%}\), compared with \(65.6\%\) for Uniform and \(58.3\%\) for Base.
This corresponds to gains of \(\mathbf{+6.3}\) and \(\mathbf{+13.6}\) percentage points, respectively, demonstrating that decision-weighted verifier learning transfers effectively to real-world robotic manipulation.

\begin{table*}[t!]
\centering
\scriptsize
\setlength{\tabcolsep}{5.0pt}
\renewcommand{\arraystretch}{1.08}
\caption{
Average success rates (\%) on LIBERO-Long with varying numbers of action candidates \(N\) sampled at test time.
\textbf{Bold} and \underline{underlined} values denote the best and second-best results, respectively.
}
\label{tab:candidate_scaling}
\resizebox{\textwidth}{!}{
\begin{tabular}{lcccc|cccc}
\toprule
\multirow{2}{*}{\textbf{Method}}
& \multicolumn{4}{c|}{\textbf{Policy: }$\boldsymbol{\pi_0}$}
& \multicolumn{4}{c}{\textbf{Policy: }$\boldsymbol{\pi_{0.5}}$} \\
\cmidrule(lr){2-5}
\cmidrule(lr){6-9}
& $N=4$ & $N=8$ & $N=16$ & $N=32$
& $N=4$ & $N=8$ & $N=16$ & $N=32$ \\
\midrule

Baseline
& \multicolumn{4}{c|}{$83.5$}
& \multicolumn{4}{c}{$93.5$} \\

KDPE~\citep{rosasco2025kdpe}
& $83.5$ & $84.5$ & $82.8$ & $82.3$
& $94.0$ & $92.7$ & $94.5$ & $95.7$ \\
MG-Select~\citep{jang2025verifier}
& $82.0$ & $80.8$ & $83.3$ & $86.3$
& $94.2$ & $95.0$ & $\underline{96.5}$ & $96.8$ \\

TACO~\citep{yang2025steering}
& $\mathbf{84.2}$ & $83.0$ & $83.5$ & $85.0$
& $\underline{94.5}$ & $94.8$ & $95.5$ & $95.7$ \\

SVM~\citep{tsao2026learning}
& $82.8$ & $\underline{85.2}$ & $\underline{85.5}$ & $\underline{86.5}$
& $94.0$ & $\underline{95.5}$ & $94.0$ & $\underline{97.0}$ \\

\ourscell{\textbf{Ours}}
& 
\ourscell{$\underline{84.0}_{\scriptscriptstyle \pm 0.5}$}
& 
\ourscell{$\mathbf{86.3}_{\scriptscriptstyle \pm 1.3}$}
& 
\ourscell{$\mathbf{87.5}_{\scriptscriptstyle \pm 1.5}$}
& 
\ourscell{$\mathbf{89.3}_{\scriptscriptstyle \pm 0.3}$}
& 
\ourscell{$\mathbf{94.8}_{\scriptscriptstyle \pm 0.3}$}
& 
\ourscell{$\mathbf{96.0}_{\scriptscriptstyle \pm 0.5}$}
& 
\ourscell{$\mathbf{97.0}_{\scriptscriptstyle \pm 0.5}$}
&
\ourscell{$\mathbf{97.5}_{\scriptscriptstyle \pm 0.5}$} 
\\

\bottomrule

\end{tabular}}
\label{tab:libero}
\end{table*}

\begin{table*}[t!]
\centering
\scriptsize
\setlength{\tabcolsep}{3.6pt}
\renewcommand{\arraystretch}{1.08}
\caption{Category-wise average success rates (\%) on RoboCasa atomic tasks. \textbf{Bold} and \underline{underlined} values denote the best and second-best results, respectively.
}
\label{tab:robocasa}
\resizebox{\textwidth}{!}{
\begin{tabular}{lccccc|ccccc}
\toprule
\multirow{2}{*}{\textbf{Method}}
& \multicolumn{5}{c|}{\textbf{Policy: }$\boldsymbol{\pi_0}$}
& \multicolumn{5}{c}{\textbf{Policy: }$\boldsymbol{\pi_{0.5}}$} \\
\cmidrule(lr){2-6}
\cmidrule(lr){7-11}
& Articulated & Control & Object-Centric & Navigation & Average
& Articulated & Control & Object-Centric & Navigation & Average \\
\midrule

Baseline
& \underline{37.2}
& 16.2
& 42.5
& 0.0
& 32.2
& 50.0
& 18.8
& 58.3
& 0.0
& 43.1 \\

KDPE
& 32.8
& 26.2
& \textbf{49.2}
& 0.0
& \underline{35.0}
& 47.8
& 20.0
& 57.5
& 0.0
& 43.3 \\

MG-Select
& 29.3
& 22.5
& 36.7
& 0.0
& 28.6
& 53.6
& \underline{33.8}
& \textbf{66.7}
& \textbf{15.0}
& \underline{51.3} \\

TACO
& 37.1
& \underline{28.8}
& 35.8
& 0.0
& 32.7
& 52.1
& 26.2
& 53.3
& 0.0
& 43.9 \\

SVM
& \underline{37.2}
& 26.2
& 40.8
& 0.0
& 33.9
& \underline{55.0}
& 20.0
& 60.8
& \textbf{15.0}
& 46.9 \\

\ourscell{\textbf{Ours}}
& $\ourscell{\mathbf{38.6}}_{\scriptscriptstyle \pm 1.2}$
& $\ourscell{\mathbf{33.8}}_{\scriptscriptstyle \pm 1.0}$
& $\ourscell{\underline{46.7}}_{\scriptscriptstyle \pm 1.6}$
& $\ourscell{0.0}_{\scriptscriptstyle \pm 0.0}$
& $\ourscell{\mathbf{38.1}}_{\scriptscriptstyle \pm 1.2}$
& $\ourscell{\mathbf{59.3}}_{\scriptscriptstyle \pm 0.9}$
& $\ourscell{\mathbf{41.2}}_{\scriptscriptstyle \pm 1.7}$
& $\ourscell{\underline{65.0}}_{\scriptscriptstyle \pm 1.0}$
& $\ourscell{\mathbf{15.0}}_{\scriptscriptstyle \pm 0.0}$
& $\ourscell{\mathbf{54.7}}_{\scriptscriptstyle \pm 1.0}$ \\

\bottomrule
\end{tabular}
}
\vspace{-0.4cm}
\end{table*}

\begin{table*}[t]
\centering
\small
\caption{
Real-robot evaluation with $\pi_{0.5}$.
\textbf{Left}: robot setup.
\textbf{Center}: the four evaluation tasks.
\textbf{Right}: task success rates (\%) over 24 evaluation episodes per task. 
}
\label{tab:real_robot}
\setlength{\tabcolsep}{4pt}
\renewcommand{\arraystretch}{1.05}

\begin{minipage}[c]{0.282\textwidth}
\vspace{0pt}
\centering
\includegraphics[width=\linewidth]{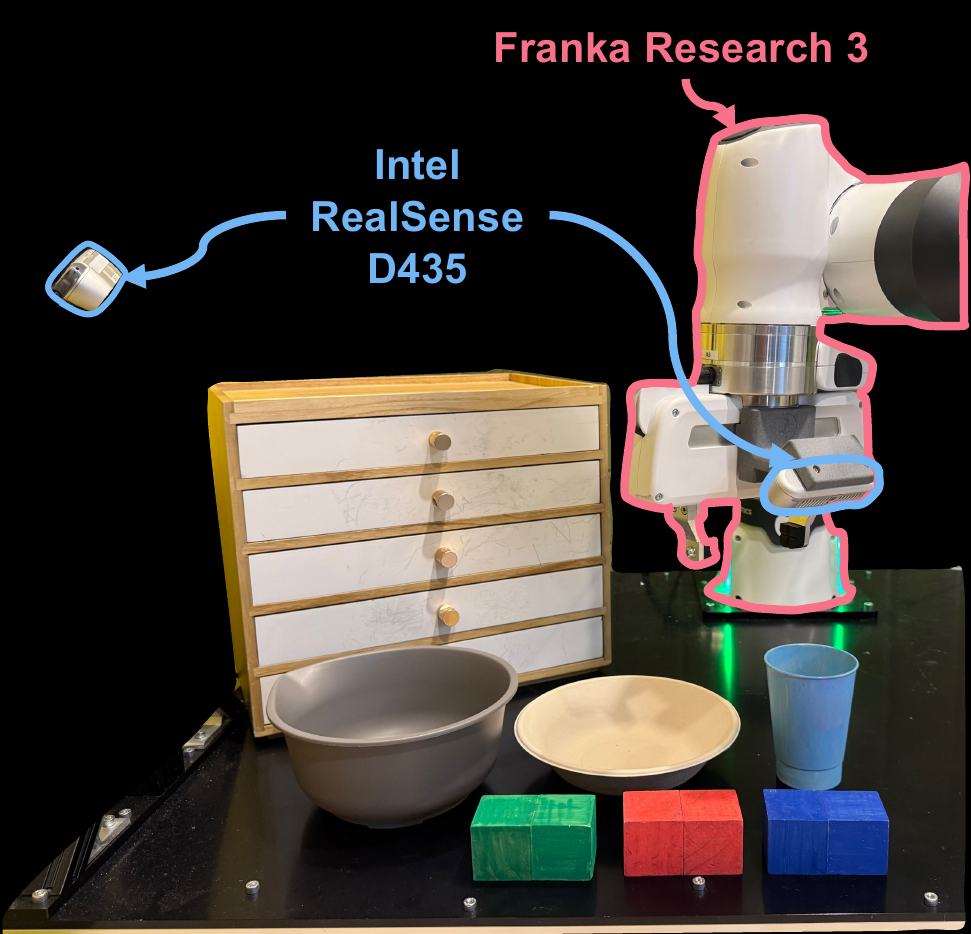}

\end{minipage}
\hfill
\begin{minipage}[c]{0.38\textwidth}
\vspace{0pt}
\centering

\begin{tabular}[t]{@{}p{0.49\linewidth}@{\hspace{2pt}}p{0.49\linewidth}@{}}
\taskimage{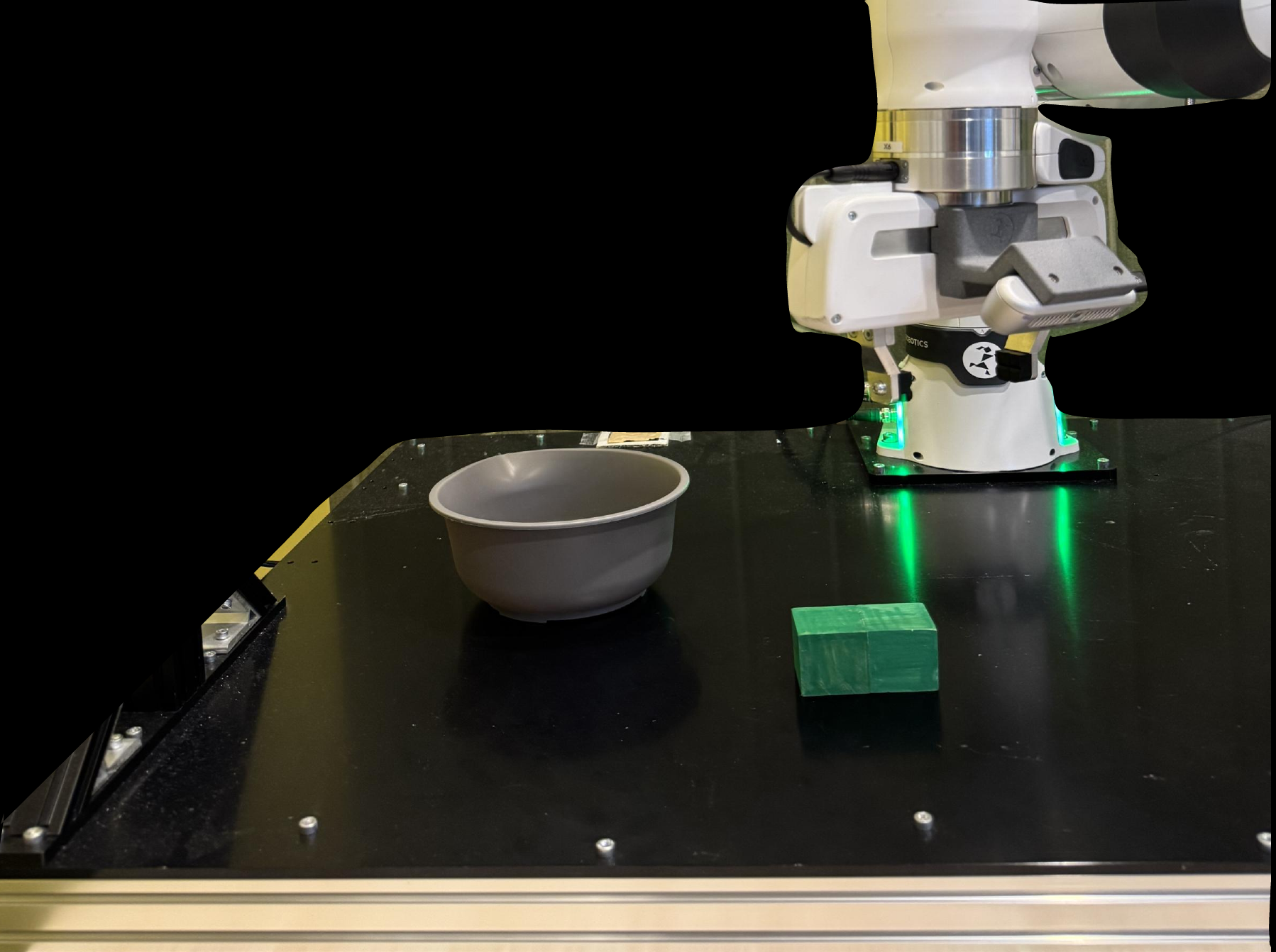}{PnP (Block)} &
\taskimage{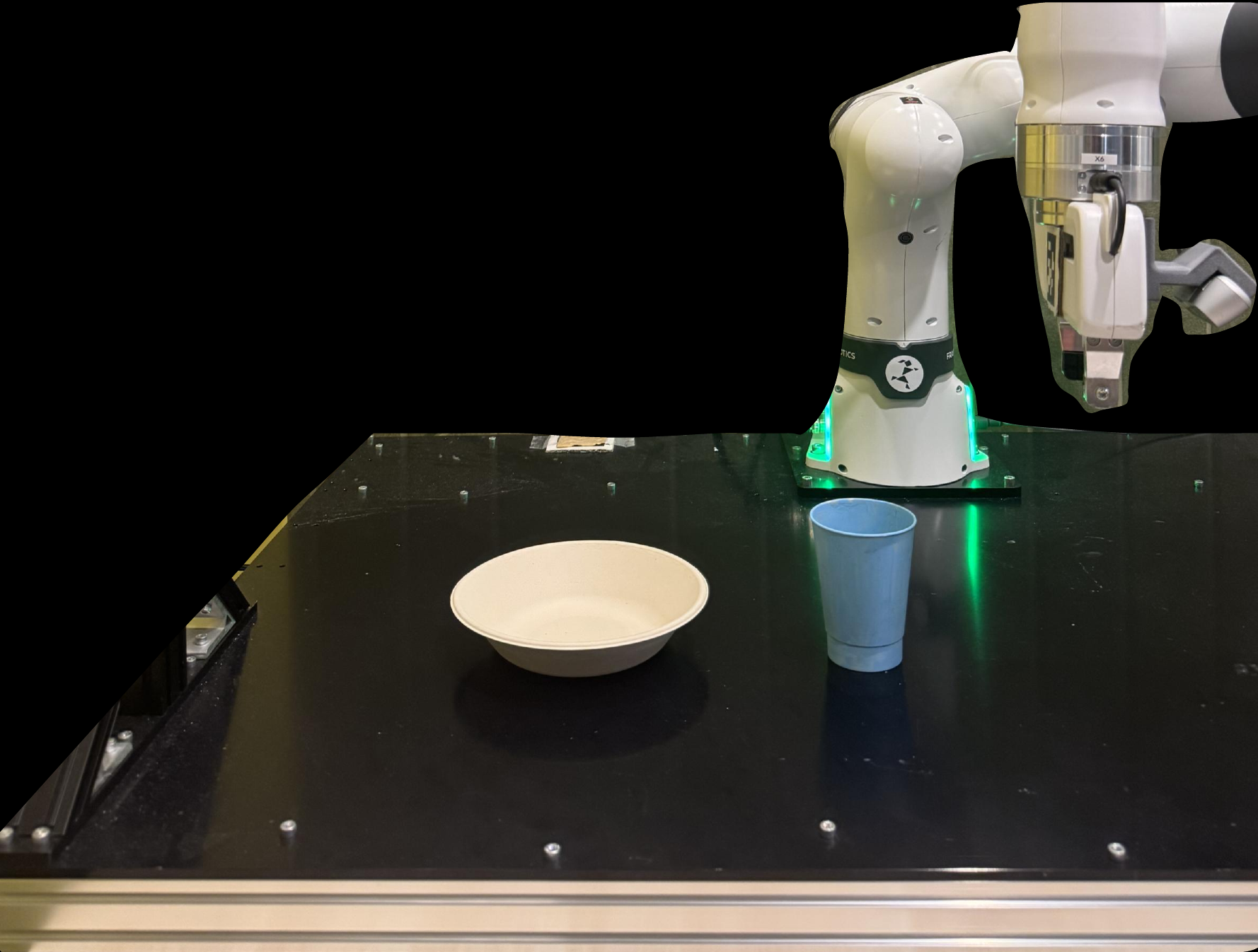}{PnP (Cup)}
\\[-2pt]
\taskimage{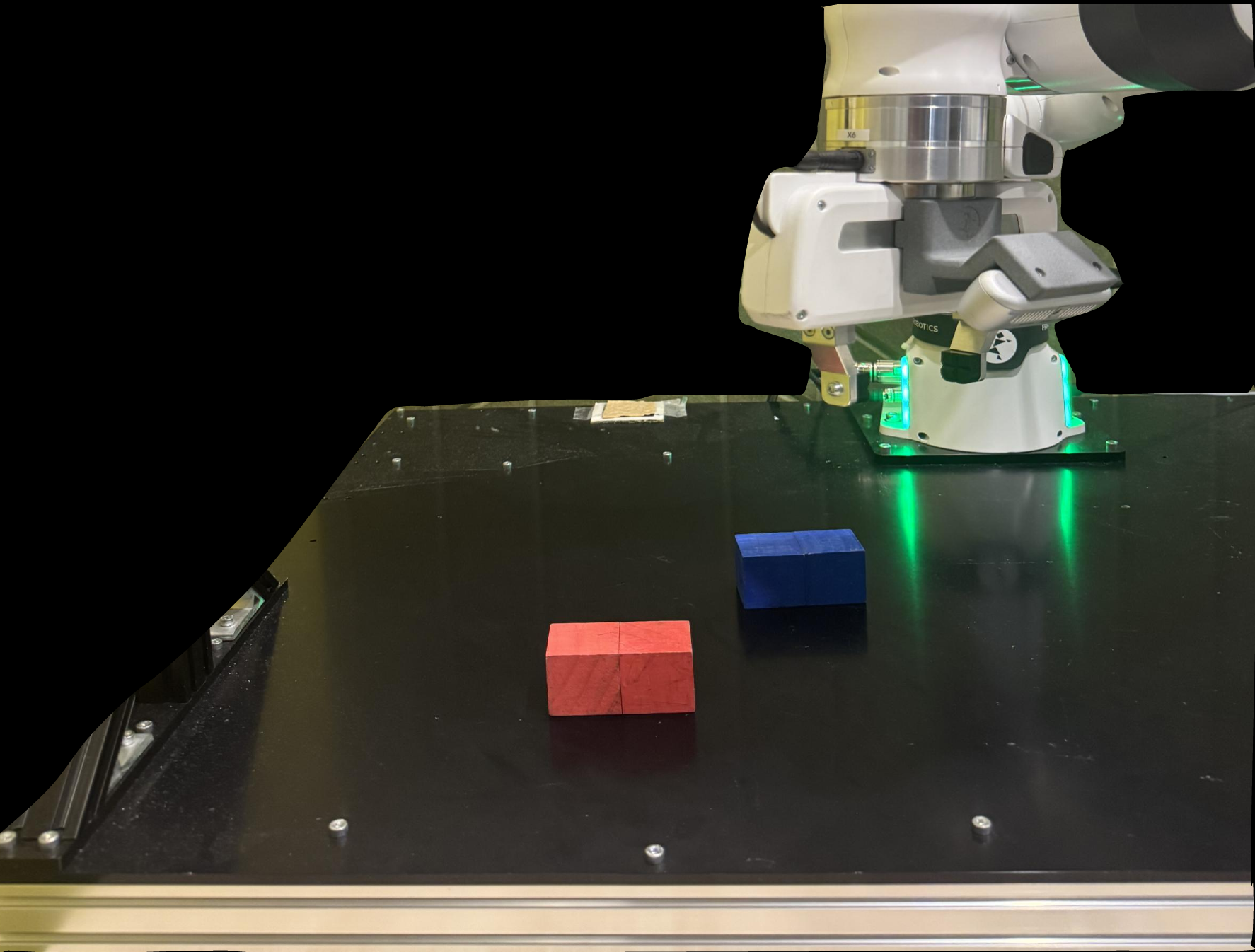}{Stack Cube} &
\taskimage{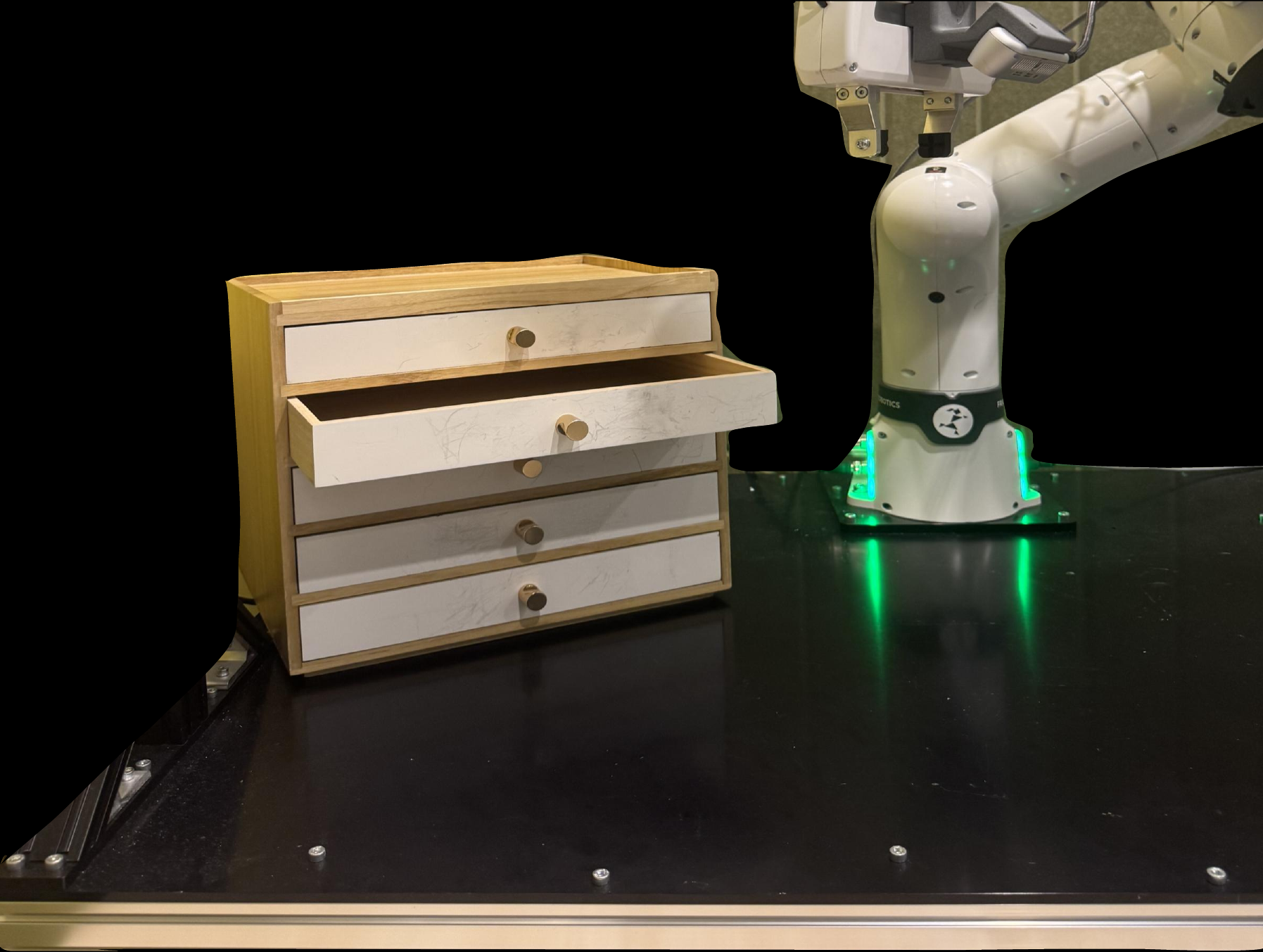}{Close Drawer}
\end{tabular}

\end{minipage}
\hfill
\begin{minipage}[c]{0.32\textwidth}
\vspace{0pt}
\centering

\resizebox{\linewidth}{!}{%
\begin{tabular}{@{}lccc}
\toprule
\textbf{Task} & \textbf{Base} & \textbf{Uniform} & \ourscell{\textbf{Ours}}\\
\midrule
PnP (block)   & 62.5 & 66.7 & \ourscell{\textbf{75.0}}\\
PnP (cup)     & 58.3 & 66.7 & \ourscell{\textbf{70.8}} \\
Stack block    & 45.8 & 58.3&  \ourscell{\textbf{66.7}}\\
Close drawer  & 66.7 &70.8 & \ourscell{\textbf{75.0}} \\
\midrule
\textbf{Average} & 58.3 & 65.6 &\ourscell{\textbf{71.9}} \\
\bottomrule
\end{tabular}
}
\end{minipage}
\vspace{-0.4cm}
\end{table*}

\begin{table*}[t]
\centering
\begin{minipage}[t]{0.32\textwidth}
    \centering
    \small
    \caption{Component analysis.}
    \label{tab:component}
    \resizebox{\linewidth}{!}{%
    \begin{tabular}{lcc}
        \toprule
        \textbf{Method} &
        $\boldsymbol{\pi_0}$ &
        $\boldsymbol{\pi_{0.5}}$ \\
        \midrule
        Baseline ($N=1$) & 32.2 & 43.1 \\
        Random selection & 33.0 & 44.1 \\
        \midrule
        Uniform training & 34.9 & 50.6 \\
        \ourscell{\textbf{Ours}} & \ourscell{38.1} & \ourscell{54.7} \\
        \bottomrule
    \end{tabular}}
\end{minipage}
\hfill
\begin{minipage}[t]{0.32\textwidth}
    \centering
    \small
    \caption{Verifier input ablation.}
    \label{tab:action_feature}
    \resizebox{\linewidth}{!}{%
    \begin{tabular}{lcc}
        \toprule
        \textbf{Input feature} &
        $\boldsymbol{\pi_0}$ &
        $\boldsymbol{\pi_{0.5}}$ \\
        \midrule
        Image \& Raw & 32.4 & 50.2 \\
        Image \& Repr. & 37.4 & 54.6  \\
        \ourscell{\textbf{Representation}} & 
        \ourscell{38.1} & \ourscell{54.7} \\
        \bottomrule
    \end{tabular}}
\end{minipage}
\hfill
\begin{minipage}[t]{0.32\textwidth}
    \centering
    \small
    \caption{Comparison with a VLM-based verifier.}
    \label{tab:vlm}
    \resizebox{\linewidth}{!}{%
    \begin{tabular}{lcc}
        \toprule
        \textbf{Method} &
        \textbf{SR (\%)} $\uparrow$ &
        \textbf{Latency (s)} $\downarrow$ \\
        \midrule
        RoboMeter & 50.1 & 0.743 \\
        \ourscell{\textbf{Ours}} & \ourscell{54.7} & \ourscell{0.001} \\
        \bottomrule
    \end{tabular}}
\end{minipage}
\end{table*}

\subsection{Quantitative Analysis}
\label{sec:abl}

In this section, we provide quantitative analyses of \textsc{DiVeR}. Unless otherwise specified, all analyses are conducted on RoboCasa with $N=16$ and report success rates (\%). Additional analyses and ablations are provided in Appendix~\ref{sec:add_analy}.

\textbf{Component analysis.}
We ablate the main components of our method by comparing: (i) the base policy with a single sample (\(N=1\)), denoted \textit{Baseline}; (ii) random selection among \(N\) sampled candidates (\textit{Random selection}); (iii) our verifier trained on action-expert representations with uniform weighting (\textit{Uniform training}); and (iv) the full decision-criticality-weighted verifier (\textit{Ours}).

As shown in~\Cref{tab:component}, \textit{random selection} provides only marginal gains over the \(N=1\) baseline, indicating that additional samples are useful only when paired with an effective selection rule.
\textit{Training the verifier uniformly} on action-expert representations improves performance by \(+2.7\) and \(+7.5\) percentage points over the baseline for \(\pi_0\) and \(\pi_{0.5}\), respectively.
This variant can be viewed as SVM with its image and raw action input replaced by action-expert representations; compared with the original SVM in~\Cref{tab:robocasa}, it further improves performance by \(+1.0\) and \(+3.7\) points, highlighting the benefit of the representation.
Finally, decision-criticality weighting (\textit{Ours}) yields an additional \(+3.2\) and \(+4.1\) points over uniform training, achieving the best overall performance.
These results show that both the action-expert representation and decision-criticality weighting contribute to effective candidate selection.

\textbf{Input feature ablation.}
In~\Cref{tab:action_feature}, we compare verifier inputs based on raw actions, image features combined with action-expert representations, and action-expert representations alone.
Using action-expert representations consistently outperforms raw-action inputs, while additionally incorporating image features provides no further benefit.
This suggests that the action-expert representation already provides a sufficiently informative, policy-conditioned feature space for candidate ranking.
We therefore use action-expert representations alone throughout our method, which also keeps the verifier lightweight and avoids additional visual feature processing.

\textbf{\textsc{DiVeR} is more effective and efficient than a VLM-based verifier.}
Using $\pi_{0.5}$, we compare \textsc{DiVeR} with RoboMeter~\citep{liang2026robometer}, a VLM-based reward model for predicting task progress and trajectory preferences, with candidate visual trajectories obtained via simulator branching.
As shown in~\Cref{tab:vlm},  \textsc{DiVeR} improves the success rate by \(\mathbf{+4.6}\) percentage points while reducing verification latency from \(0.743\)~s to \(0.001\)~s, making the verifier over \(\mathbf{700\times}\) faster.
These results demonstrate that our lightweight action-representation verifier achieves more effective candidate selection with substantially lower inference overhead.

\textbf{Can \textsc{DiVeR} generalize to unseen tasks?}

\begin{wrapfigure}{r}{0.4\columnwidth}
    \centering
    \vspace{-24pt}    \includegraphics[width=\linewidth]{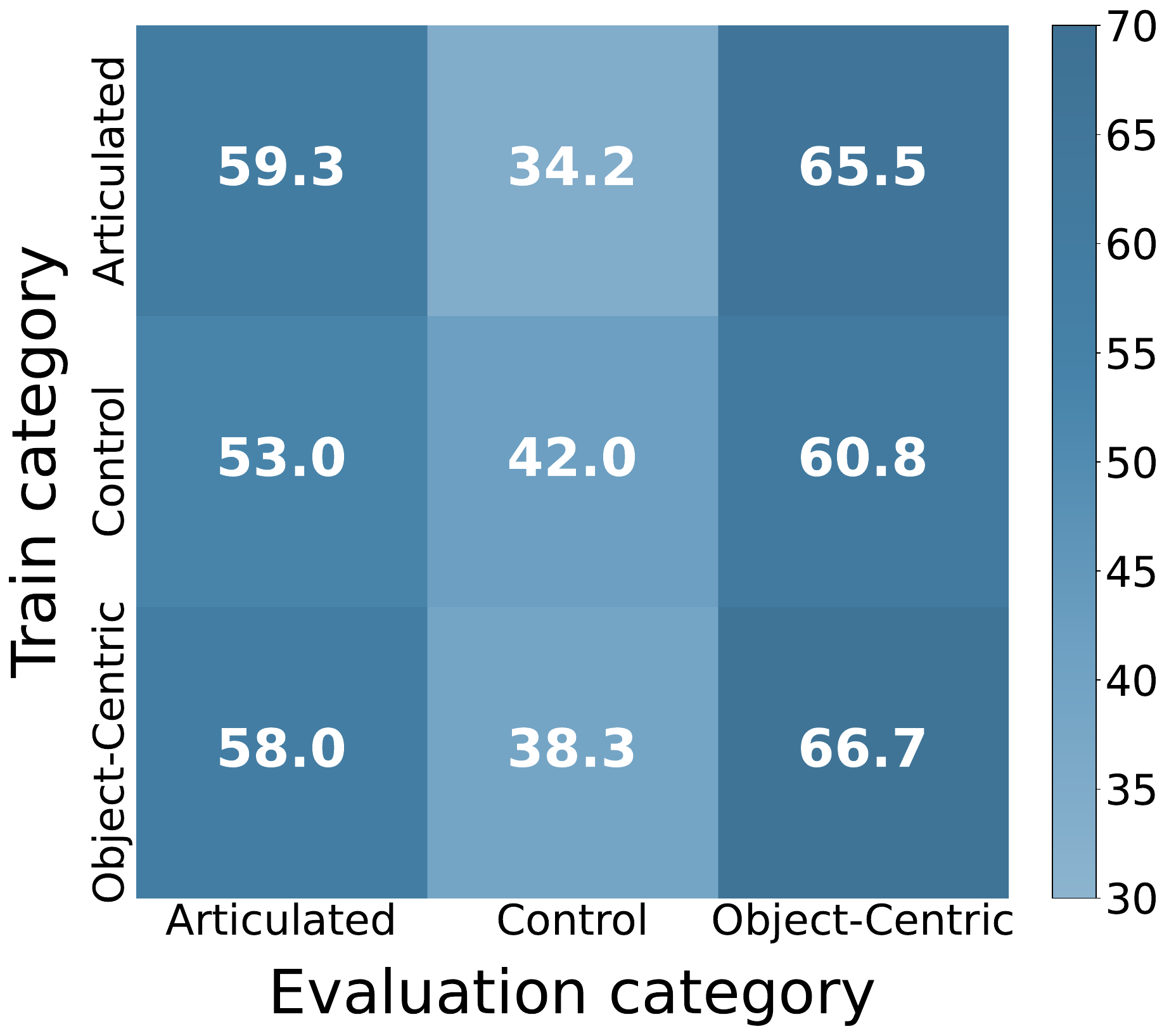}
    \vspace{-16pt}
    \caption{
    \small
    \textbf{Cross-task generalization.}
Success rates (\%) when the verifier is trained on one task category and evaluated on another.
    }
\label{fig:generalization}
    \vspace{-20pt}
\end{wrapfigure}
We evaluate zero-shot task generalization on RoboCasa with $\pi_{0.5}$ by grouping the manipulation tasks into three functional categories: articulated tasks (e.g., opening or closing), control tasks (e.g., turning a lever), and object-centric tasks (e.g., pick-and-place).
We exclude navigation tasks due to their qualitatively different behaviors.
For each source category, we train a verifier on its training tasks and evaluate it on held-out tasks from all three categories.
In~\Cref{fig:generalization}, each entry reports the success rate (\%) for a training--evaluation category pair; diagonal entries correspond to evaluation within the training category, while off-diagonal entries measure cross-category transfer.
As shown, \textsc{DiVeR} transfers effectively across task categories, indicating that the learned verifier generalizes beyond the tasks used for training.

\begin{figure}[t]
    \centering
    \includegraphics[width=\columnwidth]{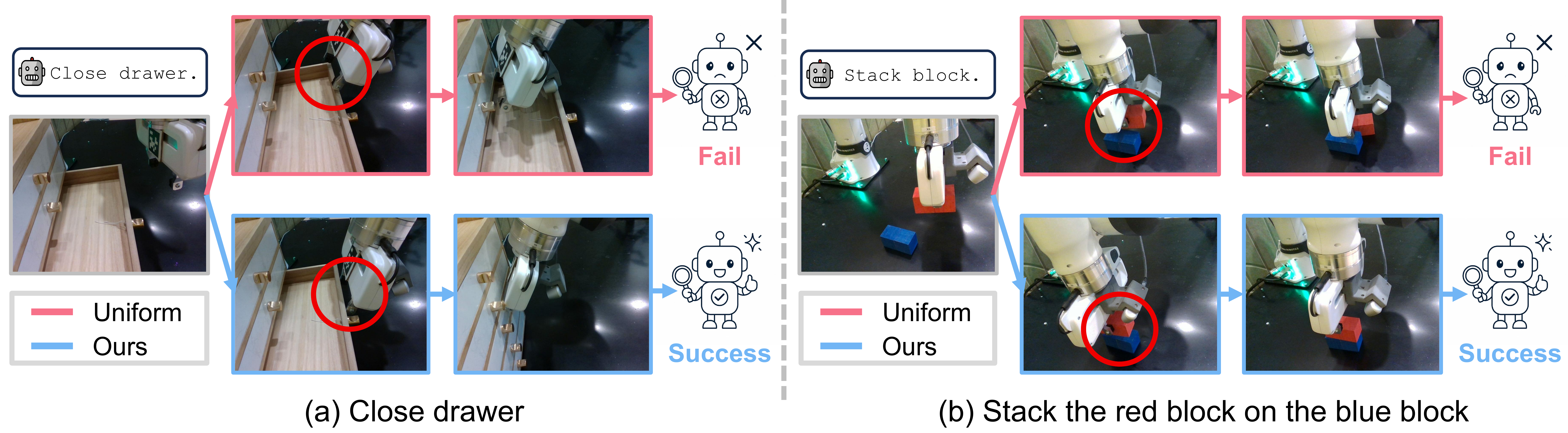}
\caption{
\textbf{Qualitative comparison of verifier-guided action selection.}
Starting from the same decision-critical state, we compare trajectories produced by a \textcolor{myred}{uniformly trained verifier} and \textcolor{myblue}{\textsc{DiVeR} (Ours)}.
In both (a) drawer closing and (b) block stacking, the two verifiers select different action candidates from the same state, leading to different downstream outcomes.
Uniform verifier training selects actions that result in task failure, whereas \textsc{DiVeR} selects more appropriate candidates and successfully completes the task.
\textcolor{red}{Red circles} highlight the resulting behaviors after these action selections.
}
    \label{fig:qualitative}
\end{figure}

\textbf{Verifier architecture.}
In~\Cref{fig:architecture}, we compare MLP, LSTM, and single-layer transformer~\citep{vaswani2017attention} verifiers under a matched parameter budget.

\begin{wrapfigure}{r}{0.35\columnwidth}
    \centering
    \vspace{-24pt}
    \includegraphics[width=\linewidth]{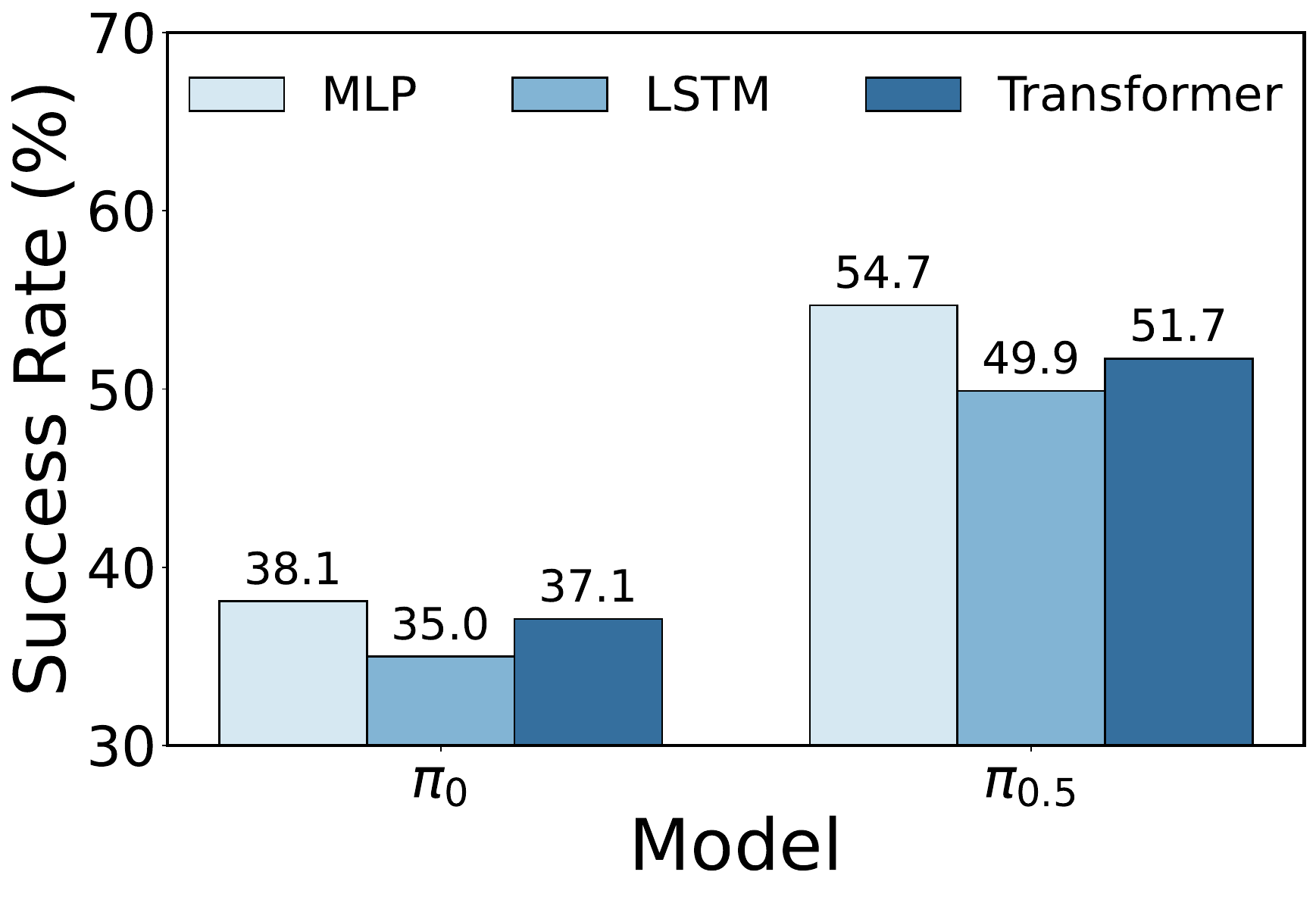}
    \vspace{-18pt}
\captionsetup{justification=centering}
\caption{\small\textbf{Verifier architecture.}}
\label{fig:architecture}
    \vspace{-8pt}
\end{wrapfigure}

All architectures use the same action-expert features, training data, and
optimization protocol, with approximately 66K trainable parameters each.
The MLP performs best across both VLA policies despite its simpler architecture.
This suggests that the frozen action-expert representation provides sufficient
information for candidate discrimination without requiring additional recurrent
or attention-based modeling.
We therefore use the MLP as our default verifier due to its strong performance
and low inference cost.

\subsection{Qualitative Analysis}

\textbf{Action selection at decision-critical states.}
We further compare the behaviors induced by a uniformly trained verifier and \textsc{DiVeR} from the same decision-critical states  on a real-world Franka Research 3 robot.
As shown in~\Cref{fig:qualitative}, the two verifiers select different action candidates from the same state, leading to different downstream outcomes.
(a) For drawer closing, the uniform verifier positions the gripper slightly above the drawer front, causing it to move through free space without making effective contact and ultimately fail.
In contrast, \textsc{DiVeR} selects an action that accurately positions the gripper against the front of the drawer, successfully completing the closing motion.
(b) Similarly, in the block-stacking task, the uniform verifier places the red block slightly beside the blue block rather than on top of it, causing the red block to fall next to the blue block and resulting in task failure.
However, \textsc{DiVeR} positions the grasped red block directly above the blue block and successfully completes the task.
These examples qualitatively illustrate how emphasizing decision-critical states during verifier training can improve candidate selection when action choice is consequential for task completion.

\section{Conclusion}

We introduced \textsc{DiVeR}, a decision-criticality-weighted verifier for VLA test-time scaling.
\textsc{DiVeR} estimates decision criticality from the dispersion of sampled action representations and places greater training emphasis on sparse states where candidate selection is most consequential.
Across LIBERO-Long, RoboCasa, and real-world manipulation tasks, \textsc{DiVeR} consistently improves task success over verifier-free and learned baselines through more effective verifier-guided Best-of-$N$ action selection, while adding negligible verifier inference overhead.
These results suggest that effective VLA test-time scaling benefits not only from generating more action candidates, but also from focusing verifier learning on the states where their differences matter most.

\section*{Acknowledgement}
We gratefully acknowledge Changdae Oh, Dahye Kim, Kinam Kim, Ilia Chelak, and Zeyi Huang for their valuable feedback.
This work was conducted during Seongheon Park's internship at
Microsoft Research Asia--Tokyo.
Seongheon Park and Sharon Li were supported in part by the
AFOSR Young Investigator Program under award FA9550-23-1-0184,
the Office of Naval Research (ONR) under award N000142612508,
the National Science Foundation under awards IIS-2237037 and
IIS-2331669, the Alfred P. Sloan Fellowship, Open Philanthropy
(now Coefficient Giving), and Schmidt Sciences Foundation.

\bibliography{iclr2027_conference}
\bibliographystyle{iclr2027_conference}

\newpage
\appendix

\vspace*{\fill}
\begin{center}
\vspace{-3cm}
    \LARGE \textbf{Appendix}
    \vspace{1em}
\end{center}

\vspace{0.2em}
\noindent\rule{\textwidth}{0.4pt}

\tableofcontents

\addtocontents{toc}{\protect\setcounter{tocdepth}{2}}
\vspace{0.2em}
\noindent\rule{\textwidth}{0.4pt}
\vspace{0.2em}
\vspace*{\fill}

\newpage

\newpage
\newpage
\newpage
\section{Algorithm}
\begin{algorithm}[h]
\caption{Decision-weighted verifier learning and Best-of-$N$ inference}
\label{alg:1}
\begin{algorithmic}[1]

\Require Frozen VLA policy $\policy$, rollout trajectories $\mathcal{T}$,
criticality samples $K$, test-time candidates $N$, training epochs $E$,
weighting parameter $\beta$

\Statex \textbf{1. Decision-Weighted Dataset Construction}
\State $\mathcal{D} \gets \emptyset$

\For{each trajectory $\tau \in \mathcal{T}$ with outcome $y_\tau$}
    \For{$t=1,\ldots,T_\tau$}
        \State Extract the executed-action representation $z_t$
        \State Sample $K$ candidate action chunks
        $\{\widetilde{\mathbf{a}}_t^{(k)}\}_{k=1}^{K}
        \sim \policy(\cdot \mid s_t)$
        \State Extract candidate representations
        $\{z_t^{(k)}\}_{k=1}^{K}$
        \State Compute criticality score $u_t$ using~\Cref{eq:action_uncertainty}
    \EndFor
    \State Compute trajectory statistics $\mu_\tau$ and $\sigma_\tau$
    over $\{u_t\}_{t=1}^{T_\tau}$
    \For{$t=1,\ldots,T_\tau$}
        \State Compute standardized criticality $\tilde{u}_t$ and weight $w_t$
        using~\Cref{eq:decision_weight}
        \State Set $y_t \gets y_\tau$
        \State Add $(z_t,y_t,w_t)$ to $\mathcal{D}$
    \EndFor
\EndFor

\Statex \textbf{2. Verifier Training}
\For{$e=1,\ldots,E$}
    \For{each class-balanced minibatch $\mathcal{B}\subset\mathcal{D}$}
        \State Update $\theta$ using the decision-weighted BCE
        in~\Cref{eq:weighted_bce}
    \EndFor
\EndFor

\Statex \textbf{3. Verifier-Guided Best-of-$N$ Inference}
\For{each decision step $t$}
    \State Sample $N$ candidate action chunks
    $\{\mathbf{a}_t^{(i)}\}_{i=1}^{N}
    \sim \policy(\cdot\mid s_t)$
    \State Extract action-expert representations
    $\{z_t^{(i)}\}_{i=1}^{N}$
    \State Score each candidate using the verifier logit
    $\widehat{f}_\theta(z_t^{(i)})$
    \State Select and execute the highest-scoring candidate
    using~\Cref{eq:bon}
\EndFor

\end{algorithmic}
\end{algorithm}

\section{Experimental Details}
\label{sec:experimental_details}

\subsection{Implementation}
\label{sec:implementation}

\textbf{Verifier architecture and representation extraction.}
We implement the verifier as a two-layer MLP operating on internal hidden representations from the VLA action expert.
The MLP uses hidden dimensions of 64 with a dropout rate of 0.3, followed by a sigmoid output layer.
We extract the final-layer action-token representations at the final denoising step and average them across action tokens. The resulting representation is used for both decision-criticality estimation and verifier training.
Ablations on the representation layer and token position are provided in~\Cref{sec:representation_extraction}.

\textbf{Training.}
We train the verifier for 30 epochs using AdamW~\citep{loshchilov2017decoupled}
with a learning rate of $10^{-3}$, weight decay of $10^{-4}$, batch size 512,
and class-balanced sampling.
For decision-weighted training, we set $\beta=0.5$ and clip
$w_t$ to $[c^{-1},c]$ with $c=2.5$ to limit the influence of
extreme criticality values.
We sample $K=32$ action candidates at each training state to estimate
decision criticality from action-representation variance.
The VLA policy remains frozen throughout training.

\textbf{Software and hardware.}
Our implementation uses PyTorch~\citep{paszke2019pytorch}.
All training and inference experiments are conducted on a single NVIDIA RTX
A6000 GPU with 48\,GB of memory.

\subsection{Hyperparameters}
\vspace{0.2cm}

\begin{table}[h!]
\centering
\caption{
Hyperparameter search space applied uniformly across all experiments.
Selected values are \textbf{bolded}.
}
\label{tab:global_param}
\setlength{\tabcolsep}{6pt}
\renewcommand{\arraystretch}{1.05}
\begin{tabular}{@{} l c @{}}
\toprule
\textbf{Hyperparameter} & \textbf{Search Space} \\
\midrule

\multicolumn{2}{@{}l}{\textit{Optimization}} \\
\quad Optimizer
    & \{Adam, \textbf{AdamW}, SGD\} \\
\quad Learning rate
    & \{1e-2, 5e-3, 1e-3, \textbf{5e-4}, 1e-4\} \\
\quad Training epochs
    & \{10, \textbf{30}, 50, 100\} \\

\midrule
\multicolumn{2}{@{}l}{\textit{Architecture}} \\
\quad MLP hidden size
    & \{\textbf{64}, 256, 512\} \\

\midrule
\multicolumn{2}{@{}l}{\textit{Decision weighting}} \\
\quad Weight temperature $\beta$
    & \{0.1, 0.25, \textbf{0.5}, 1.0, 2.0\} \\
\quad Weight clipping threshold $c$
    & \{1.5, 2.0, \textbf{2.5}, 3.0, 5.0\} \\

\bottomrule
\end{tabular}
\end{table}

We tune the verifier hyperparameters on RoboCasa using $\pi_{0.5}$ and then
fix the selected configuration for all remaining benchmarks and VLA policies.
This protocol avoids benchmark-specific tuning and evaluates whether a single
hyperparameter configuration transfers across different environments and policy
backbones.

\subsection{Dataset statistics}
\label{sec:dataset_statistics}
\vspace{0.1cm}

\begin{table}[h!]
    \centering
    \small
 \caption{
    \textbf{Dataset statistics.}
    Number of successful and failed episodes and their corresponding
    environment steps in the offline rollout datasets collected with
    (a)~\(\pi_0\) and (b)~\(\pi_{0.5}\).
    }
    \begin{subtable}[h]{0.49\linewidth}
        \centering

        \label{tab:dataset_statistics_pi0}
        \caption{\(\pi_0\)}
    \resizebox{\linewidth}{!}{
        \begin{tabular}{lrrrr}
            \toprule
            & \multicolumn{2}{c}{Success} & \multicolumn{2}{c}{Failure} \\
            \cmidrule(lr){2-3} \cmidrule(lr){4-5}
            Dataset & \# Episodes & \# Steps & \# Episodes & \# Steps \\
            \midrule
            LIBERO & 249 & 66,320 & 51 & 26,520 \\
            RoboCasa    & 184 & 60,285 & 356 & 225,600 \\
            Real Robot  & - & - & - & - \\
            \bottomrule
        \end{tabular}
        }
         
    \end{subtable}
    \hfill
    \begin{subtable}[h]{0.49\linewidth}
        \centering
     
        \label{tab:dataset_statistics_pi05}
           \caption{\(\pi_{0.5}\)}\resizebox{\linewidth}{!}{
        \begin{tabular}{lrrrr}
            \toprule
            & \multicolumn{2}{c}{Success} & \multicolumn{2}{c}{Failure} \\
            \cmidrule(lr){2-3} \cmidrule(lr){4-5}
            Dataset & \# Episodes & \# Steps & \# Episodes & \# Steps \\
            \midrule
            LIBERO & 281 & 73,140 & 19 & 9,880 \\
            RoboCasa    & 248 & 92,985 & 292 & 120,900 \\
             Real Robot  & 94 & 12,387 & 66 & 12,012 \\
             \bottomrule
        \end{tabular}
        }
        
    \end{subtable}
    \label{tab:dataset_statistics}
\end{table}

Our verifier training data are collected directly from rollouts of the VLA policy, where both successful and failed episodes arise naturally from policy execution.
As shown in~\Cref{tab:dataset_statistics}, these datasets exhibit class imbalance across policies and environments.
Despite this imbalance, our verifier can be trained effectively without artificially balancing the outcome distribution during data collection.

\subsection{Evaluation Tasks}

We evaluate our method across three settings: the LIBERO-Long benchmark,
the RoboCasa atomic-task benchmark, and a real-world manipulation setup
using a Franka Research 3 robot.
The complete task lists are provided in
\Cref{tab:real_tasks,tab:task_libero,tab:task_robocasa}.

\begin{table}[h!]
\centering
\small
\setlength{\tabcolsep}{5pt}
\renewcommand{\arraystretch}{1.08}
\caption{\textbf{Real-world robot evaluation tasks.}
We evaluate on four manipulation tasks spanning object placement,
stacking, and articulated-object interaction.}
\label{tab:real_tasks}
\begin{tabular}{cll}
\toprule
\textbf{ID} & \textbf{Task} & \textbf{Skill} \\
\midrule
1 & Place the green block in the bowl & Pick \& Place \\
2 & Place the cup on the plate & Pick \& Place \\
3 & Stack the red block on the blue block & Stacking \\
4 & Close the drawer & Articulated Manipulation \\
\bottomrule
\end{tabular}
\label{tab:task_real}
\end{table}

\begin{table*}[h!]
\centering
\small
\setlength{\tabcolsep}{5pt}
\renewcommand{\arraystretch}{1.08}
\caption{\textbf{Tasks in the LIBERO-Long benchmark.}
We evaluate on all 10 long-horizon manipulation tasks in the standard
LIBERO-Long suite.}
\label{tab:libero_tasks}
\begin{tabular}{cl}
\toprule
\textbf{ID} & \textbf{Task} \\
\midrule
1  & Put both the alphabet soup and the tomato sauce in the basket \\
2  & Put both the cream cheese box and the butter in the basket \\
3  & Turn on the stove and put the moka pot on it \\
4  & Put the black bowl in the bottom drawer of the cabinet and close it \\
5  & Put the white mug on the left plate and put the yellow and white mug on the right plate \\
6  & Pick up the book and place it in the back compartment of the caddy \\
7  & Put the white mug on the plate and put the chocolate pudding to the right of the plate \\
8  & Put both the alphabet soup and the cream cheese box in the basket \\
9  & Put both moka pots on the stove \\
10 & Put the yellow and white mug in the microwave and close it \\
\bottomrule
\end{tabular}
\label{tab:task_libero}
\end{table*}

\begin{table*}[h!]
\centering
\small
\setlength{\tabcolsep}{4.5pt}
\renewcommand{\arraystretch}{1.08}
\caption{\textbf{RoboCasa evaluation tasks.}
The 18 Atomic-Seen tasks span four broad interaction categories:
Articulated (7), Control (4), Object-Centric (6), and Navigation (1).
These categories consolidate the finer-grained atomic skills defined
by RoboCasa365.}
\label{tab:robocasa_tasks}
\begin{tabular}{cll}
\toprule
\textbf{Category} & \textbf{Task} & \textbf{RoboCasa Skill} \\
\midrule

\multirow{7}{*}{\shortstack[l]{\textbf{Articulated}}}
& Close blender lid          & Open/Close Lid \\
& Close fridge               & Open/Close Door \\
& Close toaster oven door    & Open/Close Door \\
& Open cabinet               & Open/Close Door \\
& Open drawer                & Open/Close Drawer \\
& Open stand mixer head      & Open/Close Lid \\
& Slide dishwasher rack      & Slide Rack \\

\midrule

\multirow{4}{*}{\shortstack[l]{\textbf{Control}}}
& Turn off stove             & Twist Knob \\
& Turn on sink faucet        & Turn Lever \\
& Turn on electric kettle    & Press Button \\
& Turn on microwave          & Press Button \\

\midrule

\multirow{6}{*}{\shortstack[l]{\textbf{Object-Centric}}}
& Coffee setup mug                        & Insertion \\
& Pick place: counter $\rightarrow$ cabinet & Pick \& Place \\
& Pick place: counter $\rightarrow$ stove   & Pick \& Place \\
& Pick place: drawer $\rightarrow$ counter  & Pick \& Place \\
& Pick place: sink $\rightarrow$ counter    & Pick \& Place \\
& Pick place: toaster $\rightarrow$ counter & Pick \& Place \\

\midrule

\multirow{1}{*}{\shortstack[l]{\textbf{Navigation}}}
& Navigate kitchen & Navigation \\

\bottomrule
\end{tabular}
\label{tab:task_robocasa}
\end{table*}

\section{Additional Analysis}
\label{sec:add_analy}

\subsection{Does Measured Criticality Reflect Actual Decision Consequence?}
\label{sec:criticality_validation}

\begin{wrapfigure}{r}{0.3\columnwidth}
    \centering
    \vspace{-10pt}
\includegraphics[width=\linewidth]{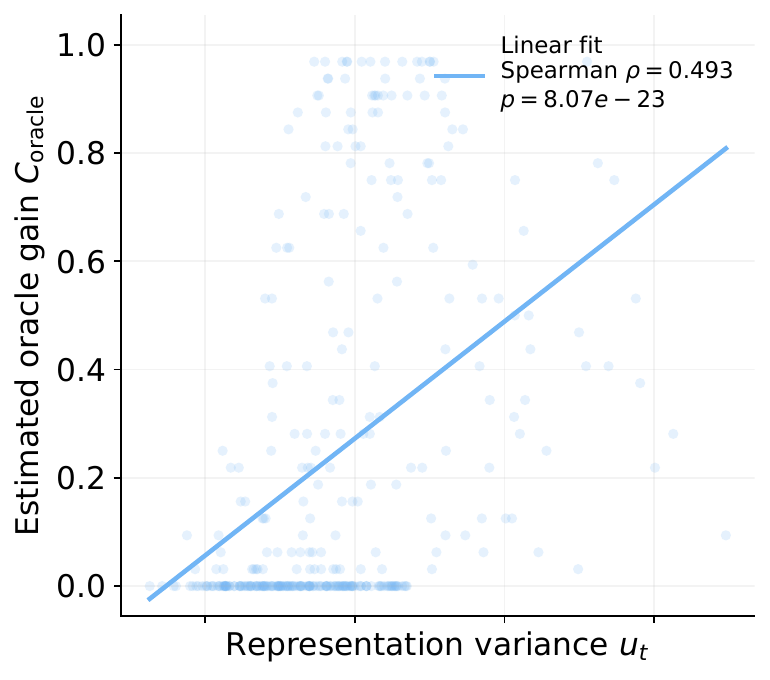}
    \vspace{-20pt}
   \caption{\small
Measured criticality $u_t$ vs. oracle criticality $C_{\mathrm{oracle}}$.
}
\label{fig:criticality_validation}
\vspace{-30pt}
\end{wrapfigure}
Our method uses action-representation variance as a rollout-free proxy for decision criticality, based on the intuition that candidate selection should matter more when plausible actions diverge.
We directly test this assumption by comparing $u_t$ with rollout-based estimates of the downstream consequence of selecting different action candidates.

We randomly sample 500 states from RoboCasa using $\pi_{0.5}$. For each saved simulator state $s_t$, we sample $K$ action candidates $\{a_t^{(k)}\}_{k=1}^{K}$ and compute the corresponding representation-based criticality $u_t$. To measure the downstream consequence of each candidate, we branch the simulator from the same state, execute $a_t^{(k)}$, and roll out the frozen base policy to estimate its success probability~\citep{wang2024math},
\[
Q(s_t,a_t^{(k)})
=
P(\mathrm{success}\mid s_t,a_t^{(k)},\pi_{0.5}).
\]

We define rollout-based oracle criticality as the potential gain from selecting the best candidate rather than choosing uniformly at random:
\[
C_{\mathrm{oracle}}(s_t)
=
\max_k Q(s_t,a_t^{(k)})
-
\frac{1}{K}\sum_{k=1}^{K}Q(s_t,a_t^{(k)}).
\]
This quantity is small when candidate selection has little influence on the final outcome and large when selecting the appropriate action is consequential. As shown in~\Cref{fig:criticality_validation}, our measured criticality $u_t$ is positively correlated with rollout-based oracle criticality (Spearman $\rho=0.493$, $p=8.09\times10^{-23}$). Together with the gains from criticality-weighted verifier learning over uniform training in~\Cref{tab:real_robot,tab:component} and the stronger downstream performance of action-representation features over raw actions in~\Cref{tab:action_feature}, these results provide complementary evidence that action-representation variance identifies states where candidate selection has meaningful downstream consequences and provides an effective signal for verifier learning.

\subsection{Autoregressive Model} \label{sec:autoregressive} \begin{wraptable}{r}{0.35\columnwidth} \centering \vspace{-16pt} \caption{ Success rate (\%) with OpenVLA using $N=4$ candidates on LIBERO-Long. } \label{tab:openvla_autoregressive} \setlength{\tabcolsep}{6pt} \renewcommand{\arraystretch}{1.05} \begin{tabular}{lc} \toprule \textbf{Method} & \textbf{Success (\%)} \\ \midrule 
Base & 52.5 \\
MGS & 54.5 \\ SVM & 55.0 \\ \ourscell{\textbf{Ours}} & \ourscell{\textbf{55.5}} \\ \bottomrule \end{tabular} \vspace{-8pt} \end{wraptable} While our main experiments focus on flow-matching VLAs, we further evaluate \textsc{DiVeR} with the autoregressive OpenVLA model to assess its applicability across different action-generation paradigms. As shown in~\Cref{tab:openvla_autoregressive}, with $N=4$ candidates, \textsc{DiVeR} improves the success rate by +$3.0$ percentage points over the pretrained policy, while outperforming MGSelect and SVM by +$1.0$ and +$0.5$ percentage points, respectively. These results indicate that our verifier-guided selection framework is not specific to flow-matching policies and can also be effectively applied to autoregressive VLAs.

\subsection{Inference Latency of Action Sampling}
\label{sec:infer_action}

\begin{wrapfigure}{r}{0.35\columnwidth}
    \centering
    \vspace{-16pt}
    \includegraphics[width=\linewidth]{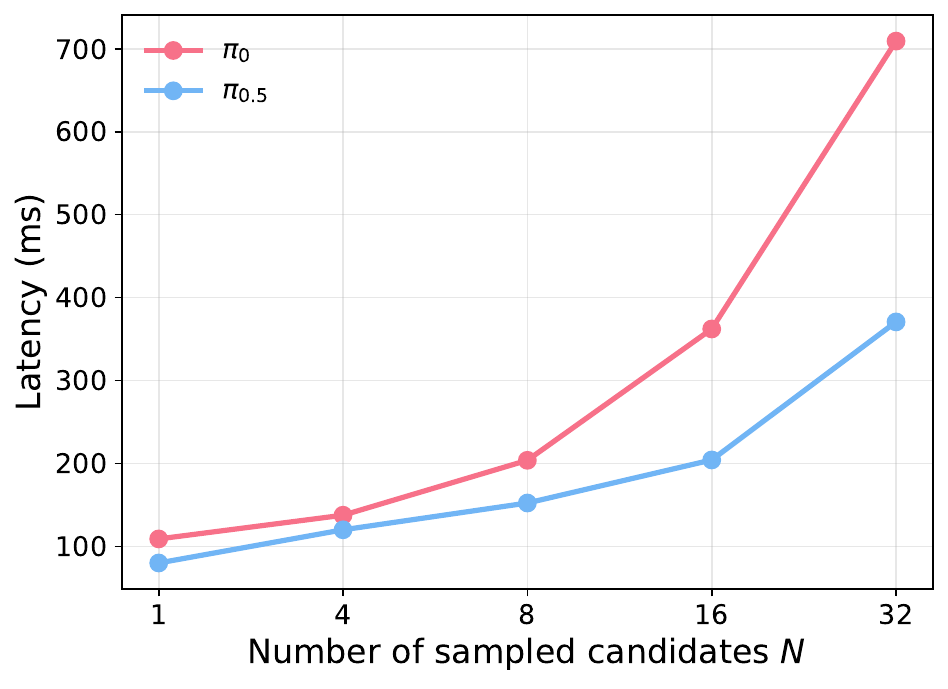}
    \vspace{-14pt}
    \caption{
    \small
    Chunk-level action sampling latency.
    }
    \label{fig:action_sampling_latency}
    \vspace{-8pt}
\end{wrapfigure}

We measure the chunk-level inference latency required to generate multiple
action candidates from the frozen VLA policy.
As shown in~\Cref{fig:action_sampling_latency}, sampling latency increases with
the number of candidates $N$, reaching approximately $360$\,ms for $\pi_0$ and
$205$\,ms for $\pi_{0.5}$ at $N=16$.
We use $N=16$ as the default inference setting, providing a practical trade-off
between candidate diversity and sampling cost.
In comparison, our lightweight verifier requires \textit{less than $0.5$\,ms} to score
each action chunk, adding minimal overhead relative to VLA action generation.
Thus, the inference cost of \textsc{DiVeR} is dominated by candidate sampling
rather than verifier evaluation.

\subsection{Effect of the Number of Samples $K$ for Criticality Estimation} \label{sec:criticality_samples}
\begin{wraptable}{r}{0.35\columnwidth} \centering \vspace{-16pt} \caption{ Effect of the number of samples $K$ used for criticality estimation. } \label{tab:criticality_samples} \setlength{\tabcolsep}{7pt} \renewcommand{\arraystretch}{1.05} \begin{tabular}{@{}cc@{}} \toprule \textbf{$K$} & \textbf{Success (\%)} \\ \midrule 4 & 52.8 \\ 16 & 53.9 \\ 32 & 54.7 \\ 64 & 55.0 \\ \bottomrule \end{tabular} \vspace{-8pt} \end{wraptable}
We study the effect of the number of action samples $K$ used to estimate state criticality on RoboCasa with $\pi_{0.5}$.
As shown in~\Cref{tab:criticality_samples}, performance improves as $K$ increases from 4 to 32, indicating that a larger candidate set provides a more reliable estimate of action disagreement. 
Increasing $K$ further to 64 yields only marginal improvement while requiring additional sampling cost.
We therefore use $K=32$ throughout our experiments, which provides a favorable trade-off between the quality of criticality estimation quality and computational cost during data collection.

\subsection{Representation Extraction}
\label{sec:representation_extraction}

\begin{table}[h!]
    \centering
    \small
    \caption{
    \textbf{Representation extraction ablation on RoboCasa with $\pi_{0.5}$.}
    Success rate (\%) under different
    (a) action-token positions and
    (b) action-expert layers.
    }
    \label{tab:representation_extraction}

    \begin{minipage}{0.7\linewidth}
        \centering

        \begin{subtable}[t]{0.46\linewidth}
            \centering
            \caption{Token position}
            \label{tab:token_position}
            \resizebox{\linewidth}{!}{
            \begin{tabular}{lc}
                \toprule
                \textbf{Position} & \textbf{Success Rate (\%)} \\
                \midrule
                First   & 53.6 \\
                Last    & 51.5 \\
                Average & \textbf{54.7} \\
                \bottomrule
            \end{tabular}}
        \end{subtable}
        \hspace{0.04\linewidth}
        \begin{subtable}[t]{0.46\linewidth}
            \centering
            \caption{Action-expert layer}
            \label{tab:representation_layer}
            \resizebox{\linewidth}{!}{
            \begin{tabular}{lc}
                \toprule
                \textbf{Layer} & \textbf{Success Rate (\%)} \\
                \midrule
                First  & 52.9 \\
                Middle & 44.1 \\
                Last   & \textbf{54.7} \\
                \bottomrule
            \end{tabular}}
        \end{subtable}

    \end{minipage}
\end{table}

We study how the choice of action-token position and action-expert layer
affects verifier performance.
This ablation is conducted on RoboCasa using $\pi_{0.5}$.
We compare representations extracted from the first, last, or averaged
action-token positions, as well as from the first, middle, or final
action-expert layers.
As shown in~\Cref{tab:representation_extraction}, averaging the action-token
representations yields the highest success rate among the token-position
variants, while the final action-expert layer performs best among the
layer choices.
We therefore use the average action-token representation from the final
action-expert layer as our default representation for both
decision-criticality estimation and verifier training.

\subsection{Visualization of Action-Expert Representations}
\label{sec:visualized_rep}

\begin{figure}[h!]
    \centering

    \begin{minipage}{0.8\columnwidth}
        \centering

        \begin{subfigure}[t]{0.5\linewidth}
            \centering
            \includegraphics[width=\linewidth]{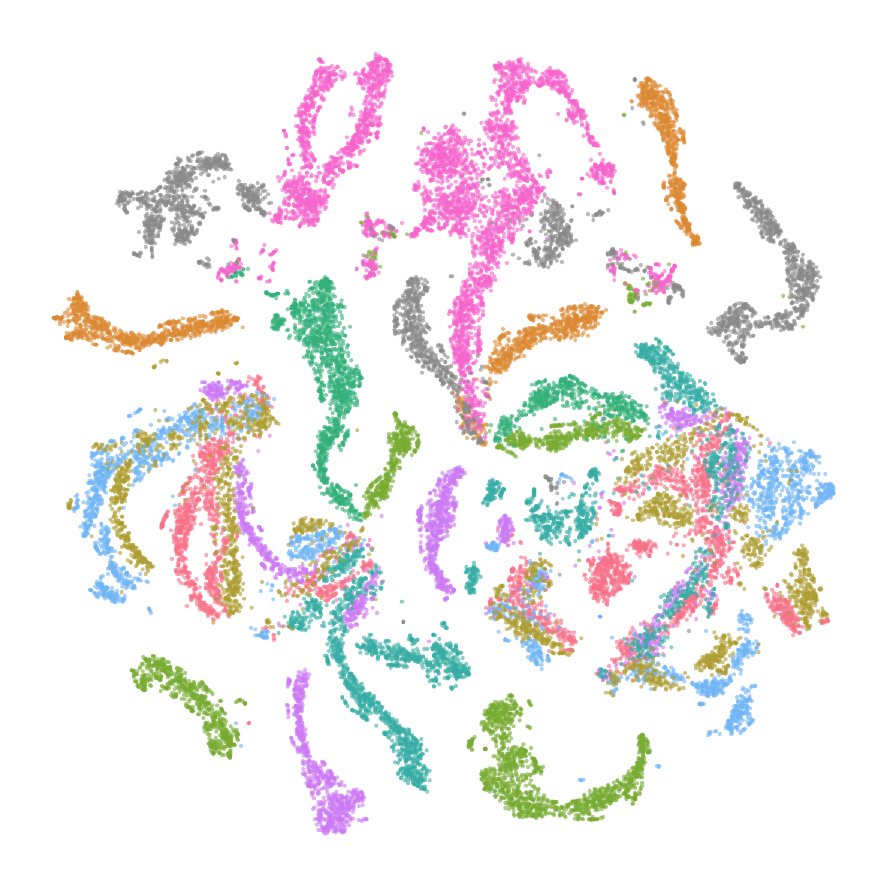}
            \caption{Task-wise structure}
            \label{fig:pi0_action}
        \end{subfigure}
        \hfill
        \begin{subfigure}[t]{0.48\linewidth}
            \centering
            \includegraphics[width=\linewidth]{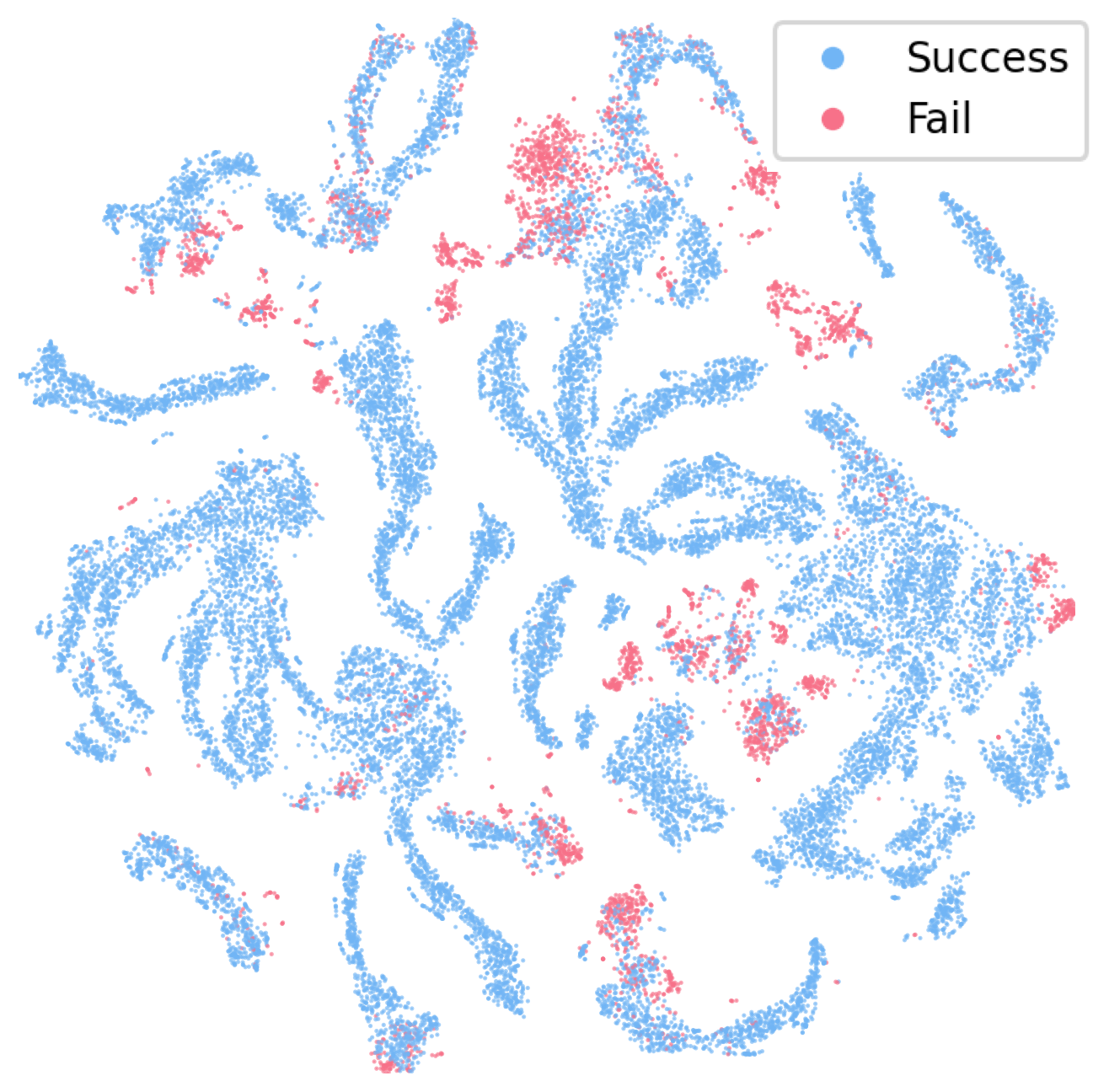}
            \caption{Success vs.\ failure}
            \label{fig:pi0_sf}
        \end{subfigure}

    \end{minipage}

    \caption{
    \textbf{t-SNE visualization of action-expert representations.}
    We visualize action-expert representations from a frozen $\pi_0$ policy over 500 LIBERO-Long episodes.
    (a) Different colors denote different tasks, revealing task-dependent structure in the representation space.
    (b) Successful and failed trajectories occupy distinct regions of the representation space.
    }
    \label{fig:action_expert}
\end{figure}

We further analyze the internal action-expert representations used by the verifier.
As shown in~\Cref{fig:action_expert}(a), representations exhibit structured task dependent organization, with samples from the same task tending to occupy similar regions of the embedding space.
This suggests that the representation retains task related context that can support a single verifier shared across multiple tasks.

We further examine whether the representation contains information related to behavioral outcomes.
For this analysis, representations after failure onset are labeled as failure representations.
As shown in~\Cref{fig:action_expert}(b), successful and failed behaviors exhibit partially distinct patterns, with failure representations concentrated in particular regions of the embedding space.
This suggests that the internal representation contains outcome related information that can provide useful features for verifier learning.

Together, these visualizations suggest that the action-expert representation contains both task related structure and outcome related variation.
While the visualization alone does not establish that the outcome related structure is fully shared across unseen tasks, it is consistent with the cross task generalization observed in~\Cref{fig:generalization}.

\section{Derivation of the Success-Visitation Ratio}
\label{app:success_visitation_derivation}

We provide the derivation showing that the logit of a binary
discriminator trained to distinguish successful from failed
state--action visitation recovers their log density ratio from~\Cref{sec:success_visitation}.
Let $\rho^{+}(s,a)$ and $\rho^{-}(s,a)$ denote the normalized
state--action visitation densities induced by successful and failed
trajectories, respectively. Assuming balanced sampling between the two
classes, we train a discriminator $f(s,a)\in(0,1)$ with the binary
cross-entropy objective
\begin{equation}
\mathcal{L}(f)
=
-\mathbb{E}_{(s,a)\sim\rho^{+}}
    [\log f(s,a)]
-\mathbb{E}_{(s,a)\sim\rho^{-}}
    [\log(1-f(s,a))].
\end{equation}

\begin{proposition}
The population-optimal discriminator satisfies
\begin{equation}
f^{\star}(s,a)
=
\frac{\rho^{+}(s,a)}
{\rho^{+}(s,a)+\rho^{-}(s,a)}.
\end{equation}
Consequently, its logit recovers the log success-to-failure visitation
density ratio:
\begin{equation}
\log\frac{f^{\star}(s,a)}
{1-f^{\star}(s,a)}
=
\log\frac{\rho^{+}(s,a)}
{\rho^{-}(s,a)}.
\end{equation}
\end{proposition}

\begin{proof}
Because the objective decomposes pointwise over state--action pairs,
we can minimize the contribution of each $(s,a)$ independently.
For a fixed $(s,a)$, let
\[
p=\rho^{+}(s,a),
\qquad
q=\rho^{-}(s,a),
\]
and write $f=f(s,a)$. The corresponding pointwise loss is
\begin{equation}
\ell(f)
=
-p\log f-q\log(1-f).
\end{equation}
Differentiating with respect to $f$ gives
\begin{equation}
\frac{\partial \ell}{\partial f}
=
-\frac{p}{f}
+
\frac{q}{1-f}.
\end{equation}
Setting the derivative to zero,
\begin{equation}
-\frac{p}{f}
+
\frac{q}{1-f}
=0,
\end{equation}
which implies
\begin{equation}
qf=p(1-f).
\end{equation}
Therefore,
\begin{equation}
f^{\star}
=
\frac{p}{p+q}
=
\frac{\rho^{+}(s,a)}
{\rho^{+}(s,a)+\rho^{-}(s,a)}.
\end{equation}
Moreover,
\begin{equation}
\frac{\partial^2 \ell}{\partial f^2}
=
\frac{p}{f^2}
+
\frac{q}{(1-f)^2}
>0,
\end{equation}
so this stationary point is the unique minimizer whenever
$p+q>0$.

Finally,
\begin{align}
\frac{f^{\star}(s,a)}
     {1-f^{\star}(s,a)}
&=
\frac{
\rho^{+}(s,a)/
(\rho^{+}(s,a)+\rho^{-}(s,a))
}{
\rho^{-}(s,a)/
(\rho^{+}(s,a)+\rho^{-}(s,a))
}
\\
&=
\frac{\rho^{+}(s,a)}
     {\rho^{-}(s,a)}.
\end{align}
Taking the logarithm of both sides yields
\begin{equation}
\log\frac{f^{\star}(s,a)}
{1-f^{\star}(s,a)}
=
\log\frac{\rho^{+}(s,a)}
{\rho^{-}(s,a)},
\end{equation}
which proves the result.
\end{proof}

\paragraph{Effect of decision weighting.}
The proposition above characterizes the unweighted BCE objective and
provides the density-ratio interpretation of our base verifier.
\textsc{DiVeR} further reweights this objective according to decision criticality,
placing greater learning emphasis on states where accurate candidate
discrimination is more consequential. Accordingly, we interpret the weighted
verifier score as a candidate-ranking score rather than an exact density-ratio
estimate.

\section{Baselines}
\label{sec:baselines}

All selection rules operate under an identical Best-of-$N$ protocol. At each decision step $t$, we
draw a single candidate set
$\{\mathbf{a}_t^{(i)}\}_{i=1}^{N}\sim\pi_{\mathrm{pre}}(\cdot\mid s_t)$ from the same frozen policy
and share it across all methods, so that differences in performance are attributable solely to the
selection rule rather than to the proposal distribution. Each method defines a scalar score
$g(s_t,\mathbf{a}_t^{(i)})\in\mathbb{R}$, and the robot executes
$\mathbf{a}_t^{\star}=\arg\max_{i\in[N]} g(s_t,\mathbf{a}_t^{(i)})$ as in~\Cref{eq:bon}.
We write $\mathbf{a}\in\mathbb{R}^{H\cdot d_a}$ for the flattened,
normalized action chunk, where $H$ is the action horizon and $d_a$ is
the per-step action dimension, and $z\in\mathbb{R}^{d}$ for the
action-expert representation defined in Section~\ref{sec:method}.
Learned baselines are trained on the same offline rollout dataset, with the same optimizer, batch
size, and class-balanced sampling as our verifier.

\paragraph{KDPE~\citep{rosasco2025kdpe}.}
KDPE is verifier-free and selects the candidate that best represents the policy's own action
distribution, discarding low-density (potentially out-of-distribution) samples. Using the $N$
candidates as the sample set, it forms a leave-one-out kernel density estimate in action-chunk space,
\begin{equation}
\widehat{p}_h(\mathbf{a}\mid s_t)
=\frac{1}{N-1}\sum_{j\neq i} K_h\!\left(\mathbf{a},\mathbf{a}_t^{(j)}\right),
\qquad
K_h(x,y)\propto\exp\!\left(-\frac{\lVert x-y\rVert_2^{2}}{2h^{2}}\right),
\label{eq:kdpe}
\end{equation}
and scores each candidate by its estimated density,
$g_{\mathrm{KDPE}}(s_t,\mathbf{a}_t^{(i)})=\widehat{p}_h(\mathbf{a}_t^{(i)}\mid s_t)$.
The bandwidth is set per state by Silverman's rule, $h=\widehat{\sigma}\,N^{-1/(d_{\mathrm{eff}}+4)}$,
where $\widehat{\sigma}$ is the mean per-coordinate standard deviation of the candidate set.
The original formulation defines the kernel on the $SE(3)$ manifold of end-effector poses; because
our policies emit joint-space action chunks, we apply an isotropic Gaussian kernel in the normalized
action space. This rule is a mode-seeking heuristic: it favors the most typical candidate and is
therefore risk-averse, but it is uninformed by task outcome and cannot prefer a rare candidate even
when the majority mode leads to failure.

\paragraph{MG-Select~\citep{jang2025verifier}.}
MG-Select is also verifier-free and ranks candidates by how strongly the conditioning information supports
them, in the spirit of classifier-free guidance. In its original autoregressive instantiation the
score is the log-likelihood gap between the conditioned and condition-masked policy,
$\log\pi_{\mathrm{pre}}(\mathbf{a}\mid s_t)-\log\pi_{\mathrm{pre}}(\mathbf{a}\mid \tilde{s}_t)$,
where $\tilde{s}_t$ denotes the observation with the conditioning signal masked out. Our
flow-matching policies expose no token likelihoods, so we realize this quantity as a kernel-density
ratio. We draw an auxiliary reference set of $M=N$ chunks from the masked-conditioning policy,
$\tilde{\mathbf{a}}_t^{(m)}\sim\pi_{\mathrm{pre}}(\cdot\mid\tilde{s}_t)$, estimate both densities
with the kernel of~\Cref{eq:kdpe}, and score
\begin{equation}
g_{\mathrm{MGS}}(s_t,\mathbf{a}_t^{(i)})
=\log\widehat{p}_h\!\left(\mathbf{a}_t^{(i)}\mid s_t\right)
-\log\widehat{p}_h\!\left(\mathbf{a}_t^{(i)}\mid \tilde{s}_t\right),
\end{equation}
which is a sample-based estimate of the pointwise mutual information between the candidate and the
conditioning. The selected chunk is the one whose trajectory-space density is most amplified by
conditioning. We use the same kernel and bandwidth rule as KDPE so that the two verifier-free
baselines differ only in the use of the masked reference set.

\paragraph{TACO~\citep{yang2025steering}.}
TACO casts test-time selection as anti-exploration: rather than seeking novelty, it prefers
candidates with high pseudo-count under the offline data distribution. Following the original
formulation, we instantiate the pseudo-count with random network distillation. A target network
$\bar{f}:\mathbb{R}^{d}\!\to\!\mathbb{R}^{m}$ is randomly initialized and frozen, and a predictor
$f_\psi$ is fit on the rollout dataset by
\begin{equation}
\min_{\psi}\ \mathbb{E}_{(s,\mathbf{a})\sim\mathcal{D}}
\left\lVert f_\psi(z)-\bar{f}(z)\right\rVert_2^{2},
\end{equation}
so that the residual is small on in-distribution state--action pairs and large elsewhere. Candidates
are scored by the negative prediction error,
$g_{\mathrm{TACO}}(s_t,\mathbf{a}_t^{(i)})=-\lVert f_\psi(z_t^{(i)})-\bar{f}(z_t^{(i)})\rVert_2^{2}$.
We fit $f_\psi$ on the same action-expert representations used by our verifier, which removes the
representation space as a confounder. The key distinction from our method is supervision: TACO
measures support under the behavior distribution and never observes trajectory outcomes, so it
cannot distinguish an in-distribution action that reliably fails from one that succeeds.

\paragraph{SVM~\citep{tsao2026learning}.}
Success Visitation Matching Verifier (SVM) is the learned verifier most closely related to ours: it discriminates successful from failed
state--action visitation using trajectory-level outcome labels. Let $\rho^{+}$ and $\rho^{-}$ denote
the visitation densities induced by successful and failed trajectories. A discriminator
$D_\phi(s,\mathbf{a})\in(0,1)$ is trained with unweighted binary cross-entropy,
\begin{equation}
\mathcal{L}_{\mathrm{SVM}}(\phi)
=-\mathbb{E}_{(s,\mathbf{a})\sim\rho^{+}}\left[\log D_\phi(s,\mathbf{a})\right]
-\mathbb{E}_{(s,\mathbf{a})\sim\rho^{-}}\left[\log\!\left(1-D_\phi(s,\mathbf{a})\right)\right],
\end{equation}
and its logit is used as the score,
$g_{\mathrm{SVM}}(s_t,\mathbf{a}_t^{(i)})=\log\frac{D_\phi}{1-D_\phi}$, which estimates the log
success-visitation ratio $\log\rho^{+}/\rho^{-}$. Following the original implementation, we encode
the current image observation with a CNN, concatenate the resulting feature with the corresponding
raw action chunk, and train an MLP discriminator on the concatenated input. SVM therefore differs
from DIVER along exactly two axes: (i) the input space (image features with raw actions, rather than
the policy's action-expert representation), and (ii) the weighting of training samples (uniform over
all visited timesteps, rather than reweighted by decision criticality). Table~\ref{tab:component}
isolates these two factors.

\section{Additional Literature Survey}

\paragraph{Structured variability in human motor control.}
Human motor variability changes systematically with task structure rather than remaining uniform throughout movement.
Simple reaching exhibits highly stereotyped trajectories~\citep{morasso1981spatial,abend1982human}, while complex object manipulation recruits a higher-dimensional repertoire of hand coordination patterns than simple grasping~\citep{todorov2004analysis,west2025kinematic}.
Related robotics work further shows that structured variation across repeated human demonstrations can reveal task-relevant interaction points~\citep{ogawara2003extraction}.
Together, these findings support our view that different parts of a trajectory are not equally informative, motivating greater emphasis on decision-critical states during verifier learning.

\paragraph{Non-uniform supervision and temporal credit assignment.}
Related work in reward modeling, failure detection, and sequential decision making highlights the value of learning signals that distinguish individual steps within a trajectory.
Process reward models provide feedback on intermediate steps rather than only final outcomes~\citep{lightman2024let,wang2024math,lu2025vla, li2025process, oh2026uncertainty}, while trajectory-supervised failure detection identifies temporally localized failure signals in VLA execution without step-level annotations~\citep{park2026hide}.
Similarly, \citet{arjona2019rudder} address temporal credit assignment by redistributing delayed rewards to earlier state--action pairs according to their estimated contributions to the return.
These approaches motivate moving beyond uniform supervision across a trajectory and emphasizing states where action selection has greater consequences for task success.

\section{Limitations and Future Work}

Our study has several limitations.
First, although we evaluate \textsc{DiVeR} in both simulation and on a real Franka Research 3 robot, the real-world experiments cover a limited set of manipulation tasks and do not fully characterize transfer across broader environments or robot embodiments.
Second, \textsc{DiVeR} selects among candidates proposed by a frozen policy, so its performance ultimately depends on the quality and diversity of the available action candidates.
Future work could explore broader real-world evaluation and stronger candidate generation mechanisms.

\end{document}